\documentclass{article}

\PassOptionsToPackage{sort,numbers}{natbib}
\usepackage[final]{neurips_2022}

\usepackage[utf8]{inputenc} 
\usepackage[T1]{fontenc}    
\usepackage{hyperref}       
\usepackage{url}            
\usepackage{booktabs}       
\usepackage{amsfonts}       
\usepackage{nicefrac}       
\usepackage{microtype}      
\usepackage{xcolor}         

\usepackage{array}
\usepackage{graphicx}
\usepackage{subfigure}
\usepackage{amsmath, bm}
\usepackage{dsfont}
\usepackage{amssymb}
\usepackage{multirow}
\usepackage{algorithm}
\usepackage{algpseudocode}
\usepackage[vlined,linesnumbered,ruled,algo2e]{algorithm2e}
\usepackage[linesnumbered,ruled,vlined]{algorithm2e}
\usepackage[flushleft]{threeparttable}

\usepackage{varwidth}
\usepackage{pifont}
\usepackage{makecell}
\usepackage{wrapfig}
\usepackage{enumitem}
\usepackage{caption}
\usepackage{dsfont}
\usepackage{soul}
\usepackage{varwidth}
\usepackage{pifont}
\usepackage{makecell}
\usepackage{enumitem}

\usepackage[hang,flushmargin]{footmisc}

\newcommand{\tabincell}[2]{\begin{tabular}{@{}#1@{}}#2\end{tabular}}

\usepackage{scalerel,stackengine}
\stackMath
\newcommand\reallywidehat[1]{%
\savestack{\tmpbox}{\stretchto{%
  \scaleto{%
    \scalerel*[\widthof{\ensuremath{#1}}]{\kern-.6pt\bigwedge\kern-.6pt}%
    {\rule[-\textheight/2]{1ex}{\textheight}}
  }{\textheight}%
}{0.5ex}}%
\stackon[1pt]{#1}{\tmpbox}%
}

\usepackage{listings}
\usepackage{color}

\usepackage{etoolbox}
\makeatletter
\AfterEndEnvironment{algorithm}{\let\@algcomment\relax}
\AtEndEnvironment{algorithm}{\kern2pt\hrule\relax\vskip3pt\@algcomment}
\let\@algcomment\relax
\newcommand\algcomment[1]{\def\@algcomment{\footnotesize#1}}
\renewcommand\fs@ruled{\def\@fs@cfont{\bfseries}\let\@fs@capt\floatc@ruled
  \def\@fs@pre{\hrule height.8pt depth0pt \kern2pt}%
  \def\@fs@post{}%
  \def\@fs@mid{\kern2pt\hrule\kern2pt}%
  \let\@fs@iftopcapt\iftrue}
\makeatother

\usepackage{amsmath}
\usepackage{amssymb}
\usepackage{mathtools}
\usepackage{amsthm}

\usepackage[capitalize,noabbrev]{cleveref}

\theoremstyle{plain}
\newtheorem{theorem}{Theorem}[section]

\newtheorem{lemma}{Lemma}[section]
\newtheorem{condition}{Condition}[theorem]

\theoremstyle{definition}
\newtheorem{definition}[theorem]{Definition}

\theoremstyle{remark}
\newtheorem{remark}[theorem]{Remark}

\newenvironment{customthm}[1]
  {\renewcommand\thetheorem{#1}\theorem}
  {\endtheorem}
\newenvironment{customlemma}[1]
  {\renewcommand\thelemma{#1}\lemma}
  {\endlemma}

\newcommand{\diag}{\mat{diag}}

\newcommand{\mat}[1]{\bm{#1}}
\newcommand{\vect}[1]{\bm{#1}}
\newcommand{\norm}[1]{\left\|#1\right\|}

\newcommand{\abs}[1]{\left|#1\right|}
\newcommand{\expect}{\mathbb{E}}

\newcommand{\variance}{\text{Var}}
\newcommand{\params}{\vect{\theta}}

\newcommand{\error}{\mathcal{E}}
\newcommand{\rhobound}{\Lambda}
\newcommand{\wbound}{\mathfrak{W}}
\newcommand{\linsub}{\mathcal{W}}

\newcommand{\activate}{\rho}

\newcommand{\relu}[1]{\sigma\left(#1\right)}

\usepackage[textsize=tiny]{todonotes}

\title{Deep Architecture Connectivity Matters for Its Convergence: A Fine-Grained Analysis}

\author{%
  Wuyang Chen\thanks{Equal Contribution.}\\
  University of Texas at Austin\\
   \And
  Wei Huang$^*$\\
  RIKEN AIP\\
  \And
  Xinyu Gong\\
  University of Texas at Austin\\
  \And
  Boris Hanin\\
  Princeton University\\
  \And
  Zhangyang Wang\\
  University of Texas at Austin\\
}

\begin{document}

\maketitle

\begin{abstract}
Advanced deep neural networks (DNNs), designed by either human or AutoML algorithms, are growing increasingly complex. Diverse operations are connected by complicated connectivity patterns, e.g., various types of skip connections. Those topological compositions are empirically effective and observed to smooth the loss landscape and facilitate the gradient flow in general. However, it remains elusive to derive any principled understanding of their effects on the DNN capacity or trainability, and to understand why or in which aspect one specific connectivity pattern is better than another. In this work, we theoretically characterize the impact of connectivity patterns on the convergence of DNNs under gradient descent training in fine granularity. By analyzing a wide network's Neural Network Gaussian Process (NNGP), we are able to depict how the spectrum of an NNGP kernel propagates through a particular connectivity pattern, and how that affects the bound of convergence rates. As one practical implication of our results, we show that by a simple filtration on ``unpromising" connectivity patterns, we can trim down the number of models to evaluate, and significantly accelerate the large-scale neural architecture search without any overhead.
Code is available at: \url{https://github.com/VITA-Group/architecture_convergence}.
\end{abstract}

\section{Introduction}

Recent years have witnessed substantial progress in designing better deep neural network architectures.
The common objective is to build networks that are easy to optimize, of superior trade-offs between efficiency and accuracy, and are generalizable to diverse tasks and datasets.
When developing deep networks, how operations (linear transformations and non-linear activations) are connected and stacked together is vital, which is studied in network's convergence~\cite{du2019gradient,zhou2020theory,zou2020global}, complexity~\cite{poole2016exponential,rieck2018neural,hanin2021deep}, generalization~\cite{chen2019much,cao2019generalization,xiao2019disentangling}, loss landscapes~\cite{li2017visualizing,fort2019large,shevchenko2020landscape}, etc.

Although it has been widely observed that the performance of deep networks keeps being improved by advanced design options,
our understanding remains limited on how a network's properties are influenced by its architectures.
For example, in computer vision, design trends have shifted from vanilla chain-like stacked layers~\cite{lecun1998gradient,krizhevsky2012imagenet,simonyan2014very} to manually elaborated connectivity patterns (ResNet~\cite{he2016deep}, DenseNet~\cite{huang2017densely}, etc.). While people observed smoothed loss landscapes~\cite{li2017visualizing}, mitigated gradient vanishing problem~\cite{balduzzi2017shattered}, and better generalization~\cite{huang2020deep}, these findings only explain the effectiveness of adding skip connections in general, but barely lead to further ``finer-grained" insight on more sophisticated composition of skip connections beyond ResNet.
Recently, the AutoML community tries to relieve human efforts and propose to automatically discover novel networks from gigantic architecture spaces~\cite{zoph2016neural,pham2018efficient}. Despite the strong performance~\cite{tan2019efficientnet,howard2019searching,xie2019exploring,tan2021efficientnetv2}, the searched architectures are often composed of highly complex (or even randomly wired) connections, leaving it challenging to analyze theoretically.

Understanding the principles of deep architecture connectivity is of significant importance.
Scientifically, this helps answer why composing complicated skip connection patterns has been such an effective ``trick" in improving deep networks' empirical  performance.
Practically, this has direct guidance in designing more efficient and expressive architectures.
To close the gap between our theoretical understandings and practical architectures, in this work we target two concrete questions:
\vspace{-0.5em}
\begin{enumerate}
    \item[\textbf{Q1:}] Can we understand the precise roles of different connectivity patterns in deep networks?
    \item[\textbf{Q2:}] Can we summarize principles on how connectivity should be designed in deep networks?
\end{enumerate}
\vspace{-0.5em}

\begin{wrapfigure}{r}{62mm}
\vspace{-1em}
\includegraphics[scale=0.23]{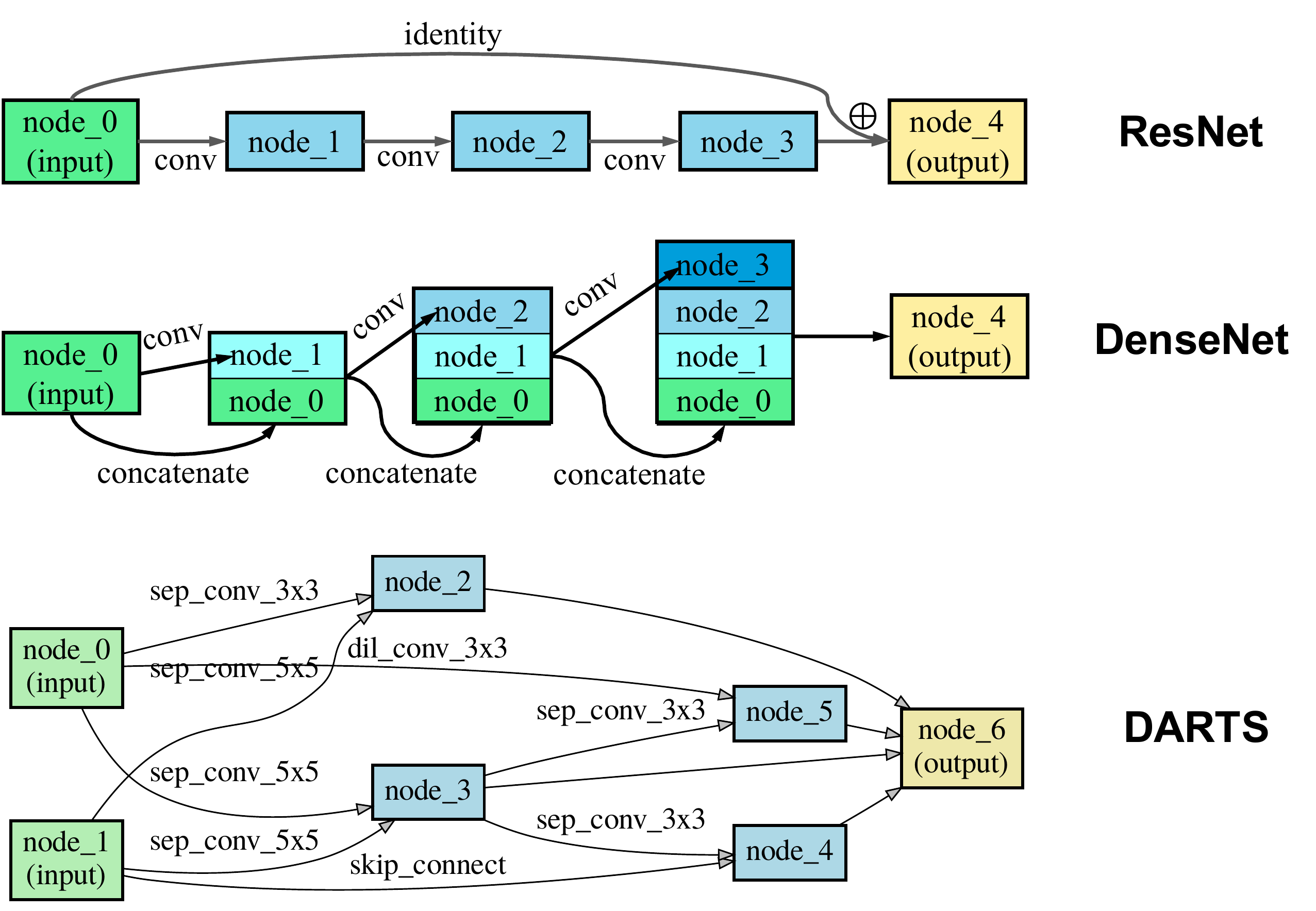}
\centering
\captionsetup{font=small}
\caption{Connectivity patterns of deep networks are growingly more complex, including ResNet~\cite{he2016deep}, DenseNet~\cite{huang2017densely}, and architectures sampled from the DARTS search space~\cite{liu2018darts} commonly used in neural architecture search (NAS).}
\label{fig:teaser}
\vspace{-1em}
\end{wrapfigure}

One standard way to study how operations and connections affect the network is to analyze the model's convergence under gradient descent~\cite{du2018gradient,allen2019convergence,du2019gradient}.
Here, we systematically study the relationship between the connectivity pattern of a neural network and the bound of its convergence rate.
A deep architecture can be viewed as a directed acyclic computational graph (DAG), where feature maps are represented as nodes and operations in different layers are directed edges linking features.
Under this formulation, by analyzing the spectrum of the Neural Network Gaussian Process (NNGP) kernel, we show that the bound of the convergence rate of the DAG is jointly determined by the number of unique paths in a DAG and the number of parameterized operations on each path. Note that, although several prior arts \cite{mellor2020neural,abbdelfattah2020zero,chen2020tenas,chen2022auto} explored deep learning theories to predict the promise of architectures, their indicators were developed for the general deep network \textit{functions}, not for fine-grained characterization for specific deep architecture \textit{topology patterns}. Therefore, their correlations for selecting better architecture topologies are only \textit{empirically observed}, but not \textit{theoretically justified}. In contrast, our fine-grained conclusion is theory-ground and also empirically verified.

Based on this conclusion, we present two intuitive and practical principles for designing deep architecture DAGs: the ``effective depth'' $\bar{d}$ and the ``effective width'' $\bar{m}$. Experiments on diverse architecture benchmarks and datasets demonstrate that $\bar{d}$ and $\bar{m}$ can jointly distinguish promising connectivity patterns from unpromising ones.
As a practical implication, our work also suggests a cheap ``plug-and-play'' method to accelerate the neural architecture search by filtering out potentially unpromising architectures at almost zero cost, before any gradient-based architecture evaluations.
Our contributions are summarized below:\vspace{-0.5em}
\begin{itemize}[leftmargin=*]
    \item We first theoretically analyze the convergence of gradient descent of diverse neural network architectures, and find the connectivity patterns largely impact their bound of convergence rate.
    \item From the theoretical analysis, we abstract two practical principles on designing the network's connectivity pattern: ``effective depth'' $\bar{d}$ and ``effective width'' $\bar{m}$.
    \item Both our convergence analysis and principles on effective depth/width are verified by experiments on diverse architectures and datasets. Our method can further significantly accelerate the neural architecture search without introducing any extra cost.
\end{itemize}
\vspace{-1em}


\section{Related works}
\vspace{-0.5em}

\subsection{Global convergence of deep networks}
\vspace{-0.5em}

Many works analyzed the convergence of networks trained with
gradient descent~\cite{jacot2018neural,zhang2019learning,hanin2019finite,chang2020provable}.
Convergence rates were originally explored for wide two-layer neural networks~\cite{venturi2018spurious,soltanolkotabi2018theoretical,song2018mean,li2018learning,du2018gradient,mei2019mean,arora2019fine}.
More recent works extended the analysis to deeper neural networks~\cite{allen2019convergence,du2019gradient,lu2020mean,zou2020gradient,zhou2020theory,zou2020global,huang2020neural} and showed that over-parameterized networks can converge to the global minimum with random initialization if the width (number of neurons) is polynomially large as the number of training samples.
In comparison, we expand such analysis to bridge the gap between theoretical understandings of general neural networks and practical selections of better ``fine-grained" architectures.

\subsection{Residual structures in deep networks}
\vspace{-0.5em}
Architectures for computer vision have evolved from plain chain-like stacked layers~\cite{lecun1998gradient,krizhevsky2012imagenet,simonyan2014very} to elaborated connectivity patterns (ResNet~\cite{he2016deep}, DenseNet~\cite{huang2017densely}, etc.). 
With the development of AutoML algorithms, novel networks of complicated operations/connections were discovered.
Despite their strong performance~\cite{tan2019efficientnet,howard2019searching,xie2019exploring,tan2021efficientnetv2}, these architectures are often composed of highly complex connections and are hard to analyze.
\cite{shu2019understanding} defined the depth and width for complex connectivity patterns, and studied the impacts of network topologies in different cases.
To better understand these residual patterns, people tried to unify the architectures with different formulations. One seminal way is to represent network structures into graphs, then randomly sample different architectures from the graph distribution and empirically correlate their generalization with graph-related metrics~\cite{xie2019exploring,you2020graph}. Other possible ways include geometric topology analysis~\cite{bu2020depth}, network science~\cite{bhardwaj2021does}, etc.
For example,~\cite{bhardwaj2021does} proposed NN-Mass by analyzing the network’s layer-wise dynamic isometry, and then empirically linked to the network’s convergence.
However, none of those works directly connect the convergence rate analysis to different residual structures.

\subsection{Theory-guided design of neural architectures}
\vspace{-0.5em}
Particularly related to our work is an emerging research direction, that tries to connect recent deep learning theories to guide the design of novel network architectures.
The main idea is to find theoretical indicators that have strong correlations with network's training or testing performance.
\cite{mellor2020neural} pioneered a training-free NAS method, which empirically leveraged sample-wise activation patterns to rank architectures.
\cite{park2020towards} leveraged the network's NNGP features to approximate its predictions.
Different training-free indicators were evaluated in~\cite{abbdelfattah2020zero}, and the ``synflow'' measure \cite{tanaka2020pruning} was leveraged as the main ranking metric.
\cite{chen2020tenas} incorporated two theory-inspired metrics with supernet pruning as the search method.
However, these works mainly adapted theoretical properties of the general deep neural network \textit{function}: their correlations with the concrete network architecture \textit{topology} are only \textit{empirically observed}, but not \textit{theoretically justified}.

\vspace{-0.5em}
\section{Topology matters: convergence analysis with connectivity patterns} \label{sec:convergence}
\vspace{-0.5em}

In this section, we study the convergence of gradient descent for deep networks, whose connectivity patterns can be formulated as small but representative direct acyclic graphs (DAGs). Our goal is to compare the convergence rate bounds of different DAGs, and further establish links to their connectivity patterns, leading to abstracting design principles.

\subsection{Problem setup and architectures notations}

\begin{wrapfigure}{r}{55mm}
\vspace{-5.5em}
\includegraphics[scale=0.3]{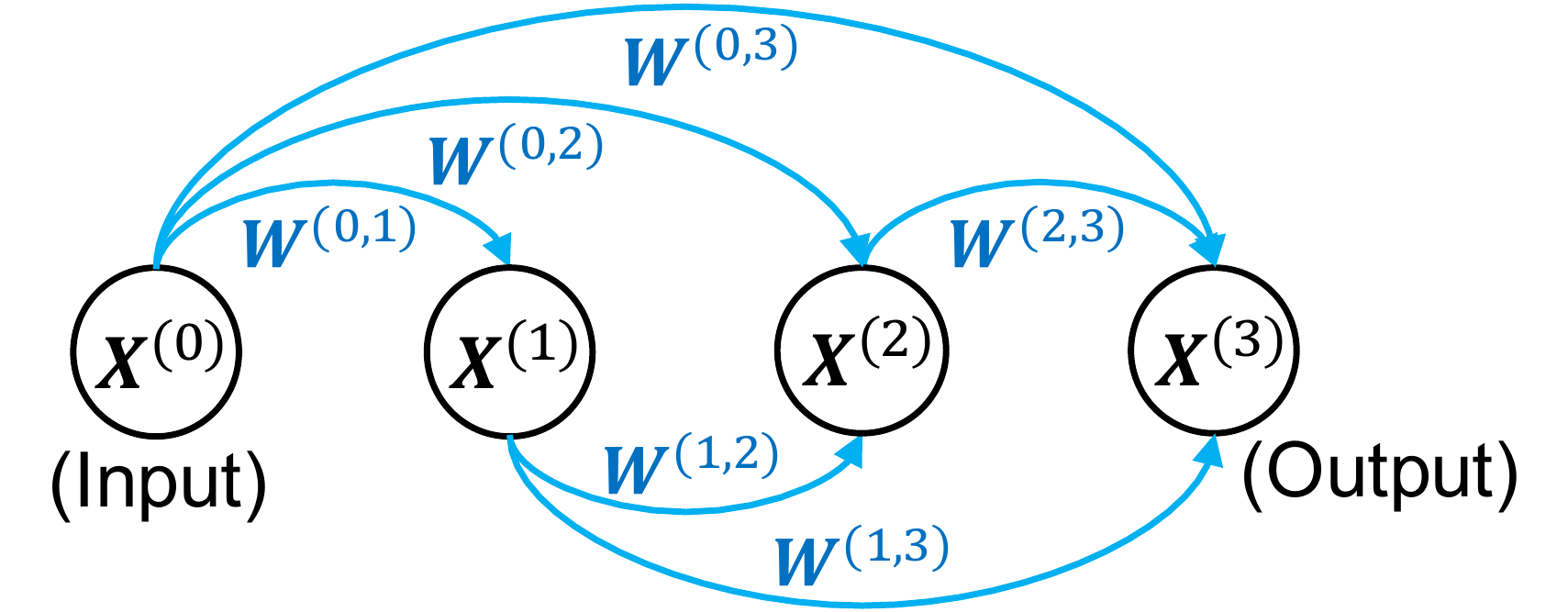}
\centering
\captionsetup{font=small}
\vspace{-0.5em}
\caption{A network represented as a direct acyclic graph (DAG). $\bm{X}^{(h)}$ is the feature map (node). $\bm{W}$ is the operation (edge). $h \in [0, H]$, here $H = 3$. If we remove some edges or set some $\bm{W}$ as skip-connections (i.e., identity mappings), how would the convergence of this DAG be affected?}
\label{fig:dag_setting}
\vspace{-0.5em}
\end{wrapfigure}

We consider a network's DAG as illustrated in Figure~\ref{fig:dag_setting}.
$\bm{X}^{(h)}~~(h \in [0, H])$ is the feature map (node) and $\bm{W}$ is the operation (edge). $\bm{X}^{(0)}$ is the input node, $\bm{X}^{(H)}$ is the output node ($H=3$ in this case), and $\bm{X}^{(1)}, \cdots, \bm{X}^{(H-1)}$ are intermediate nodes. The DAG is allowed to be fully connected: any two nodes could be connected by an operation.
Feature maps from multiple edges coming to one vertex will be directly summed up.
This DAG has many practical instances: for example, it is used in NAS-Bench-201~\cite{dong2020bench}, a popular NAS benchmark.
The forward process of the network in Figure \ref{fig:dag_setting} can be formulated as:
{\small
\begin{equation}
\begin{aligned}
    \bm{X}^{(1)} &= \rho(\bm{W}^{(0,1)} \bm{X}^{(0)}) \\
    \bm{X}^{(2)} &= \rho(\bm{W}^{(0,2)} \bm{X}^{(0)}) + \rho (\bm{W}^{(1,2)} \bm{X}^{(1)}) \\
    \bm{X}^{(3)} &= \rho(\bm{W}^{(0,3)} \bm{X}^{(0)}) + \rho(\bm{W}^{(1,3)} \bm{X}^{(1)}) + \rho(\bm{W}^{(2,3)} \bm{X}^{(2)}) \\
    \bm{u} &= \bm{a}^\top \bm{X}^{(3)}.
\end{aligned}
\end{equation}
}

Feature $\bm{X}^{(s)} \in \mathbb{R}^{m\times 1}$, where $m$ is the absolute width of an edge (i.e. number of neurons), and $s \in \{1, 2, 3\}.$
$\bm{a}$ is the final layer and $\bm{u}$ is the network's output.
We consider three candidate operations for each edge: a linear transformation followed by a non-linear activation, or a skip-connection (identity mapping), or a broken edge (zero mapping):
{\small
\begin{equation}
    \bm{W}^{(s,t)}
    \begin{cases}
    = \bm{0} & \text{zero} \\
    = \bm{I}^{m\times m} & \text{skip-connection} \\
     \sim \mathcal{N}(\bm{0}, \bm{I}^{m\times m}) & \text{linear transformation} \\
    \end{cases},
    \quad
    \rho(x) =  
    \begin{cases}
    0 & \text{zero} \\
    1 & \text{skip-connection} \\
    \sqrt{\frac{c_\sigma}{m}}\sigma(x), & \text{linear transformation} \\
    \end{cases}.
\end{equation}
}

$s, t \in \{1, 2, 3\}$, $\mathcal{N}$ is the Gaussian distribution, and $\sigma$ is the activation function.
${c_{\sigma}=\left(\expect_{x\sim N(0,1)}\left[\sigma(x)^2\right]\right)^{-1}}$ is a scaling factor to normalize the input in the initialization phase.

Consider the Neural Network Gaussian Process (NNGP) in the infinite-width limit~\cite{lee2017deep}, we define our NNGP variance as $\bm{K}^{(s)}_{ij} = \langle \bm{X}^{(s)}_i, \bm{X}^{(s)}_j \rangle$ and NNGP mean as $\bm{b}^{(s)}_i = \mathbb{E}[\bm{X}^{(s)}_i]$, $i, j \in 1, \cdots, N$ for $N$ training samples in total.
Both $\bm{K}_{ij}$ and $\bm{b}_i$ are taken the expectation over the weight distributions.

\subsection{Preliminary: bounding the network's linear convergence rate via NNGP spectrum}
\label{sec:background}

Before we analyze different connectivity patterns (Section~\ref{sec:propagation_analysis}), we first give the linear convergence of our DAG networks (Theorem~\ref{thm:linear_convergence}) and also show the guarantee of the full-rankness of $\lambda_{\min}(\bm{K}^{(H)})$ (Lemma~\ref{lem:full_rank}).
For a sufficiently wide neural network, its bound of convergence rate to the global minimum can be governed by the NNGP kernel. The linear convergence rate for a deep neural network of a DAG-like connectivity pattern is shown as follows:
\begin{theorem}[Linear Convergence of DAG]
    \label{thm:linear_convergence}
    Consider a DAG of $H$ nodes and $P_H$ end-to-end paths.
    At $k$-th gradient descent step on $N$ training samples, with MSE loss $\mathcal{L}(k)=\frac{1}{2} \big \|\bm{y}-\bm{u}(k) \big \|_{2}^{2}$, suppose the learning rate {$\eta=O\left(\frac{\lambda_{\min}\left(\bm{K}^{(H)}\right)}{(NP_H)^{2}} 2^{O(H)}\right)$ and the number of neurons per layer $m=\Omega\left(\max \left\{\frac{(NP_H)^{4}}{\lambda_{\min}^{4}\left(\bm{K}^{(H)}\right)},\frac{NP_HH}{\delta}, \frac{(NP_H)^{2} \log \left(\frac{HN}{\delta}\right)2^{O(H)}}{\lambda_{\min}^{2}\left(\bm{K}^{(H)}\right)}\right\}\right)$}, we have
    \begin{equation}
        \big \|\bm{y}-\bm{u}(k) \big \|_{2}^{2} \leq\left(1-\frac{\eta \lambda_{\min}(\bm{K}^{(H)})}{2}\right)^{k} \big\| \bm{y}-\bm{u}(0) \big \|_{2}^{2},
    \end{equation}
    where $P_H$ is number of end-to-end paths from $\bm{X}^{(0)}$ to $\bm{X}^{(H)}$.
\end{theorem}

\begin{remark}
    In our work, we will use a fixed small learning rate $\eta$ across different network architectures, to focus our analysis on the impact of $\lambda_{\min}(\bm{K}^{(H)})$. This is motivated by the widely adopted setting in popular architecture benchmarks~\cite{radosavovic2019network,dong2020bench} that we will follow in our experiments.
\end{remark}

The above theorem shows that the convergence rate is bounded by the smallest eigenvalue of  $\bm{K}^{(H)}$, and also indicates that networks of larger least eigenvalues will more likely to converge faster.
This requires that the $\bm{K}^{(H)}$ has full rankness, which is demonstrated by the following lemma:
\begin{lemma}[Full Rankness of $\bm{K}^{(H)}$]
    \label{lem:full_rank}
    Suppose $\sigma(\cdot)$ is analytic and not a polynomial function. If no parallel data points, then $\lambda_{\min} \left(\bm{K}^{(H)}\right) > 0$.
\end{lemma}

\subsection{How does NNGP propagate through DAG?}
\label{sec:propagation_analysis}

Now we are ready to link the DAG's connectivity pattern to the bound of its convergence rate.
Although different DAGs are all of the linear convergence under gradient descent, they are very likely to have different bounds of convergence rates.
Finding the exact mapping from $\lambda_{\min}(\bm{K}^{(0)})$ to $\lambda_{\min}(\bm{K}^{(H)})$ will lead us to a fine-grained comparison between different connectivity patterns.

First, for fully-connected operations, we can obtain the propagation of the NNGP variance and mean between two consecutive layers:

\begin{lemma}[Propagation of $\bm{K}$]
\label{lem:propagation}
Define the propagation as $\bm{K}^{(l)} = f(\bm{K}^{(l-1)})$ and $\bm{b}^{(l)} = g(\bm{b}^{(l-1)})$. When the edge operation is a linear transformation, we have:
{\small
\begin{equation}
\begin{aligned}
     \bm{K}^{(l)}_{ii} & = f(\bm{K}^{(l-1)}_{ii}) = \int  \mathcal{D}_z c_\sigma \sigma^2 \bigg(\sqrt{\bm{K}_{ii}^{(l-1)}}z \bigg)  \\ 
 \bm{K}^{(l)}_{ij} & = f(\bm{K}^{(l-1)}_{ij}) = \int \mathcal{D}_{z_1} \mathcal{D}_{z_2} c_\sigma  \sigma\bigg( \sqrt{\bm{K}^{(l-1)}_{ii} } z_1 \bigg) \sigma \bigg(\sqrt{\bm{K}^{(l-1)}_{jj}}( \bm{C}_{ij}^{(l-1)} z_1 + \sqrt{1-(\bm{C}_{ij}^{(l-1)})^2} z_2) \bigg)   \\
\bm{C}_{ij}^{(l)} & = \bm{K}_{ij}^{(l)}/\sqrt{ \bm{K}_{ii}^{(l)} \bm{K}_{jj}^{(l)}}\\ 
\bm{b}^{(l)}_{i} & = g(\bm{b}^{(l-1)}_{i}) = \int  \mathcal{D}_{z} \sqrt{c_\sigma} \sigma \big(\sqrt{ \bm{K}_{ii}^{(l)}} z \big)
\end{aligned}
\end{equation}
}
where $z$, $z_1$ and $z_2$ are independent
standard Gaussian random variables. Besides,
$\int \mathcal{D}_z = \frac{1}{\sqrt{2\pi}} \int dz e^{-\frac{1}{2}z^2}$ is the measure for a normal distribution.
\end{lemma}

\begin{remark}
    Since $\bm{b}^{(l)}_{i}$ is a constant, it will not affect our analysis and we will omit it below.
\end{remark}
\begin{remark}
    With centered normalized inputs, $\bm{K}^{(0)}_{ii} = 1$.
\end{remark}

Our last lemma to prepare states that we can bound the smallest eigenvalue of the NNGP kernel of $N \times N$ by its $2 \times 2$ principal submatrix case.
\begin{lemma}
\label{lem:lambda_min_2x2}
For a positive definite symmetric matrix $\bm{K} \in \mathbb{R}^{N \times N}$, the smallest eigenvalue is bounded by the smallest eigenvalue of its $2 \times 2$ principal sub-matrix.
\begin{equation}
\lambda_{\min}(\bm{K}) \le \min_{i \neq j} \lambda_{\min} \begin{bmatrix}
     \bm{K}_{ii}  & \bm{K}_{ij}\\
     \bm{K}_{ji} & \bm{K}_{jj} 
 \end{bmatrix}
\end{equation}
\end{lemma}

\begin{figure*}[t!]
\includegraphics[scale=0.26]{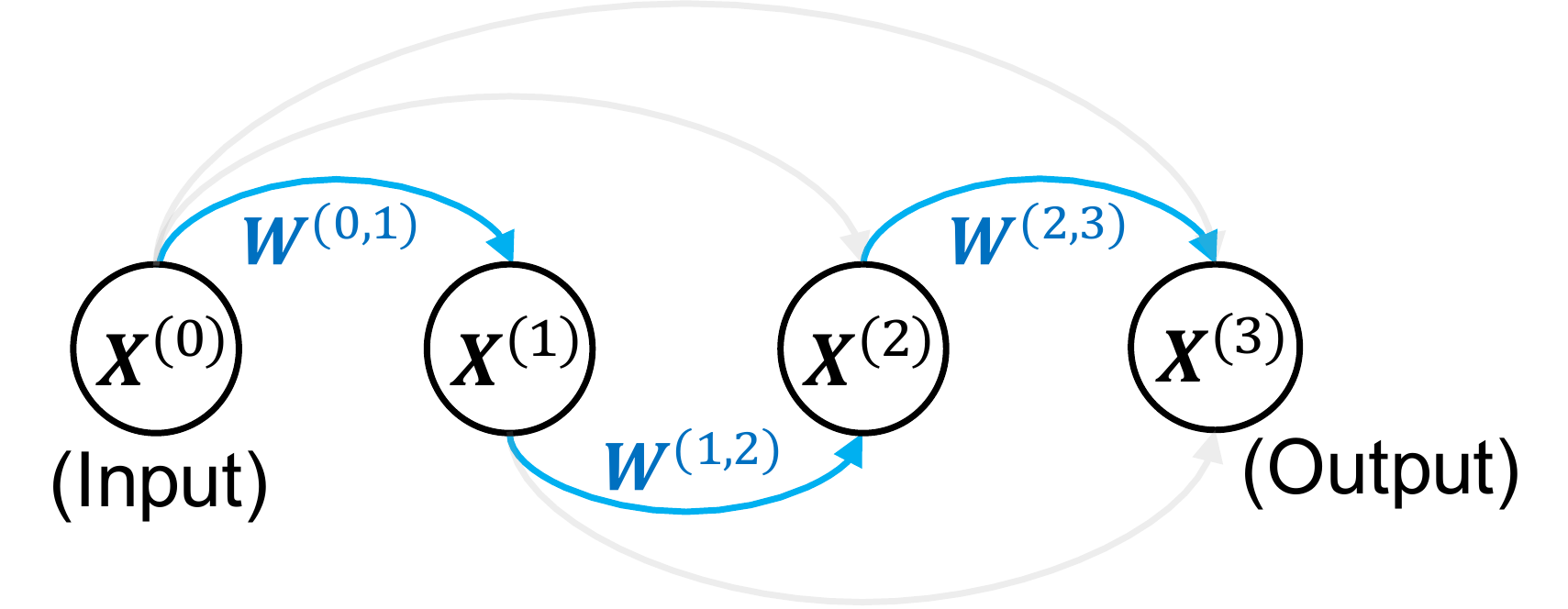}
\includegraphics[scale=0.26]{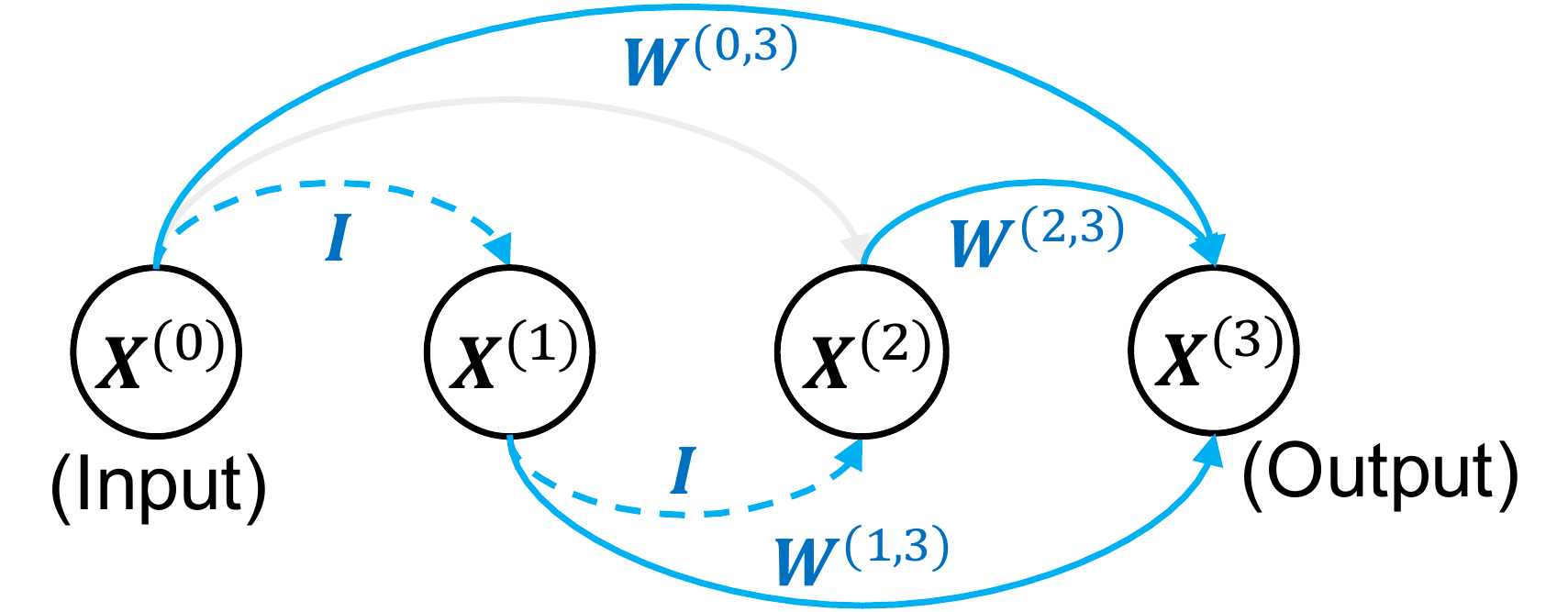}
\includegraphics[scale=0.26]{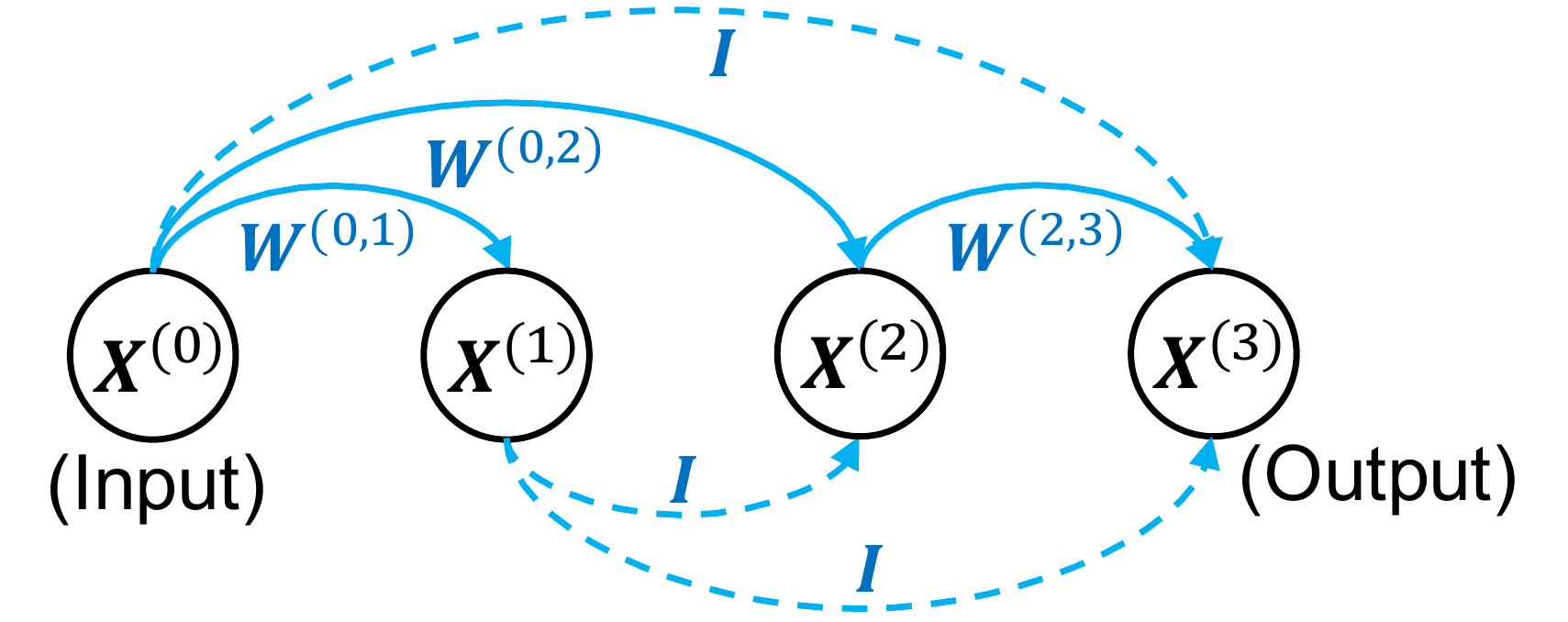}
\vspace{-0.5em}
\centering
\captionsetup{font=small}
\caption{Three example DAGs. Solid blue arrows are parameterized operations (e.g. fully-connected layer). Dashed arrows are non-parameterized operations (e.g. skip-connections). Removed edges are in light grey. See Section~\ref{sec:propagation_analysis} for their convergence analysis. Left: sequential connectivity (DAG\#1). Middle: parallel connectivity (DAG\#2). Right: mixed of sequential and parallel connectivity (DAG\#3).}
\label{fig:dag_examples}
\vspace{-1em}
\end{figure*}

Now we can analyze the smallest eigenvalue $\lambda_{\min}(\bm{K}^{(H)})$ for different DAGs. Note that here we adopt ReLU activation
for three reasons: 1) Analyzing ReLU with the initialization at the edge-of-chaos~\cite{hayou2019impact} exhibits promising results; 2) Previous works also imply that the convergence rate of ReLU-based networks depends on $\lambda_0$~\cite{du2018gradient}; 3) ReLU is the most commonly used activation in practice. Specifically, we set $c_\sigma = 2$ for ReLU~\cite{du2019gradient,hayou2019impact}.
In Figure~\ref{fig:dag_examples}, we show three representative DAGs ($H = 3$). 
For a fair comparison, each DAG has three linear transformation layers, i.e., they have the same number of parameters and mainly differ in how nodes are connected. After these three case studies, we will show a more general rule for the propagation of $\bm{K}^{(0)}$.

\paragraph{DAG\#1 (sequential connections).} In this case, $\bm{W}^{(0,1)}, \bm{W}^{(1,2)}, \bm{W}^{(2,3)} = $ linear transformation, and $\bm{W}^{(0,2)}, \bm{W}^{(0,3)}, \bm{W}^{(1,2)} = $ zero (broken edge).
There is only \ul{one unique path} connecting from $\bm{X}^{(0)}$ to $\bm{X}^{(3)}$, with \ul{three parameterized operations} on it.
\begin{equation}
  \begin{aligned}
        \bm{K}^{(1)} = f(\bm{K}^{(0)}) \quad \quad
        \bm{K}^{(2)} = f(\bm{K}^{(1)}) \quad \quad
        \bm{K}^{(3)} = f(\bm{K}^{(2)})
    \end{aligned}
\end{equation}

By Lemma~\ref{lem:lambda_min_2x2}, we calculate the upper bound of the least eigenvalue of $\bm{K}^{(3)}$, denoted as $\lambda_{\text{dag1}}$:
\begin{equation}
\begin{aligned}
& \lambda_{\min} \begin{bmatrix}
     \bm{K}^{(3)}_{ii}  & \bm{K}^{(3)}_{ij}\\
     \bm{K}^{(3)}_{ji} & \bm{K}^{(3)}_{jj} 
 \end{bmatrix}  = \lambda_{\min} \begin{bmatrix}
     \bm{K}^{(2)}_{ii}  & f(\bm{K}^{(2)}_{ij}) \\
     f(\bm{K}^{(2)}_{ji})   & \bm{K}^{(2)}_{jj} 
 \end{bmatrix} 
  = \lambda_{\min} \begin{bmatrix}
     \bm{K}^{(1)}_{ii}  & f(f(\bm{K}^{(1)}_{ij})) \\
     f(f(\bm{K}^{(1)}_{ji}))   & \bm{K}^{(1)}_{jj} 
 \end{bmatrix} \\
 & = \lambda_{\min} \begin{bmatrix}
     \bm{K}^{(0)}_{ii}  & f(f(f(\bm{K}^{(0)}_{ij}))) \\
     f(f(f(\bm{K}^{(0)}_{j,i})))   & \bm{K}^{(0)}_{jj} 
 \end{bmatrix} 
 = 1 - f^3(\bm{K}^{(0)}_{ij}) \equiv \lambda_{\text{dag1}}, \label{eq:lambda_min_dag1}
 \end{aligned}
\end{equation}
where we denote the function composition $f^3(\bm{K}^{(0)}_{ij}) = f(f(f(\bm{K}^{(0)}_{ij})))$.

\paragraph{DAG\#2 (parallel connections).} In this case, $\bm{W}^{(0,3)}, \bm{W}^{(1,3)}, \bm{W}^{(2,3)} = $ linear transformation, $\bm{W}^{(0,1)}, \bm{W}^{(1,2)} = $ skip-connection, and  $\bm{W}^{(0,2)} =$ zero (broken edge).
There are \ul{three unique paths} connecting from $\bm{X}^{(0)}$ to $\bm{X}^{(3)}$, with \ul{one parameterized operation} on each path.

\begin{equation}
\begin{aligned}
\bm{K}^{(2)} & = \bm{K}^{(1)} = \bm{K}^{(0)} \\
\bm{K}^{(3)} & = f(\bm{K}^{(0)}) + f(\bm{K}^{(1)}) +f(\bm{K}^{(2)}) = 3f(\bm{K}^{(0)})
\\
\end{aligned} 
\end{equation}
where $\bm{C}$ is a constant matrix with all entries being the same. Therefore:
\begin{equation}
\begin{aligned}
& \lambda_{\min} \begin{bmatrix}
     \bm{K}^{(3)}_{ii}  & \bm{K}^{(3)}_{ij}\\
     \bm{K}^{(3)}_{ji} & \bm{K}^{(3)}_{jj} 
 \end{bmatrix}  = \lambda_{\min} \begin{bmatrix}
     3 \bm{K}^{(0)}_{ii} & 3 f(\bm{K}^{(0)}_{ij})  \\
     3 f(\bm{K}^{(0)}_{ji}) & 3 \bm{K}^{(0)}_{jj}
 \end{bmatrix} \\
 & = 3 \lambda_{\min} \begin{bmatrix}
     \bm{K}^{(0)}_{ii}  & f(\bm{K}^{(0)}_{ij}) \\
     f(\bm{K}^{(0)}_{ji})  & \bm{K}^{(0)}_{jj} 
 \end{bmatrix} = 3(1 - f(\bm{K}^{(0)}_{ij})) \equiv \lambda_{\text{dag2}}. \label{eq:lambda_min_dag2}
 \end{aligned}
\end{equation}

\paragraph{DAG\#3 (mixture of sequential and parallel connections).} In this case, $\bm{W}^{(0,1)}, \bm{W}^{(0,2)}, \bm{W}^{(2,3)} = $ linear transformation, $\bm{W}^{(0,3)}, \bm{W}^{(1,2)}, \bm{W}^{(1,3)} = $ skip connection.
There are \ul{four unique paths} connecting from $\bm{X}^{(0)}$ to $\bm{X}^{(3)}$, with \ul{0/1/2 parameterized operations} on each path.

\begin{equation}
\begin{aligned}
\bm{K}^{(1)} &= f(\bm{K}^{(0)}) \quad \quad
\bm{K}^{(2)} = \bm{K}^{(1)} + f(\bm{K}^{(0)})\\
\bm{K}^{(3)} & = f(\bm{K}^{(2)})+ \bm{K}^{(1)} + \bm{K}^{(0)} = \bm{K}^{(0)} + f(\bm{K}^{(0)}) + f(2f(\bm{K}^{(0)}))  
\end{aligned}
\end{equation}
Then we have:
\begin{equation}
\begin{aligned}
& \lambda_{\min} \begin{bmatrix}
     \bm{K}^{(3)}_{ii}  & \bm{K}^{(3)}_{ij}\\
     \bm{K}^{(3)}_{j,i} & \bm{K}^{(3)}_{jj} 
 \end{bmatrix}  = \lambda_{\min} \begin{bmatrix}
     4 \bm{K}^{(0)}_{ii} &  \tilde{f}(\bm{K}^{(0)}_{ij}) \\
    \tilde{f}(\bm{K}^{(0)}_{j,i}) & 4 \bm{K}^{(0)}_{jj}
 \end{bmatrix} 
= 4 - \tilde{f}(\bm{K}^{(0)}_{j,i}) \equiv \lambda_{\text{dag3}}, \label{eq:lambda_min_dag3}
 \end{aligned}
\end{equation}
where $\tilde{f}(\bm{K}^{(0)}_{ij}) = \bm{K}^{(0)}_{ij} + f(\bm{K}^{(0)}_{ij}) + f(2f(\bm{K}^{(0)}_{ij})) $

\textbf{Conclusion:} by comparing Eq.~\ref{eq:lambda_min_dag1},~\ref{eq:lambda_min_dag2}, and~\ref{eq:lambda_min_dag3}, 
we can show the rank of three upper bounds of least eigenvalues as: $\lambda_{\text{dag1}} < \lambda_{\text{dag2}} < \lambda_{\text{dag3}}$.
Therefore, the bound of the convergence rate of DAG\#3 is better than DAG\#2, and DAG\#1 is the worst. See our Appendix~\ref{appendix:dag_lambda_min_conclusion} in supplement for detailed analysis.

\paragraph{General rules of propagation from $\bm{K}^{(0)}$ to $\bm{K}^{(H)}$.}
\vspace{-0.5em}
\begin{itemize}[leftmargin=*]
    \item Diagonal elements of $\bm{K}^{(H)}$ is determined by the \ul{number of unique paths} from $\bm{X}^{(0)}$ to $\bm{X}^{(H)}$.
    \item Off-diagonal elements of $\bm{K}^{(H)}$ is determined by the \ul{number of incoming edges of $\bm{X}^{(H)}$} and the \ul{number of parameterized operations} on each path.
\end{itemize}
Thus, we could simplify our rules as:
\vspace{-0.5em}
\begin{equation}
    \lambda_{\min}(\bm{K}^{(H)}) \leq \min_{i \neq j} \bigg( P - \sum_{p=1}^P f^{d_p} \big(\bm{K}^{(0)}_{ij} \big) \bigg),
    \label{eq:rule}
\end{equation}
where $P$ is number of end-to-end paths from $\bm{X}^{(0)}$ to $\bm{X}^{(H)}$, $d_p$ is the number of linear transformation operations on the $p$-th path, and $f^{d_p}$ indicates a $d_p$-power composition of $f$.
This summary is based on the propagation of the NNGP variance. 
With this rule, we can quickly link the relationship between the smallest eigenvalue of $\bm{K}^{(H)}$ and $\bm{K}^{(0)}$.
\vspace{-0.5em}

\section{Practical principle of connectivity: effective depth and effective width} \label{sec:effective_depth_width}
\vspace{-0.5em}

Eq.~\ref{eq:rule} precisely characterizes the impact of a network's topology on its convergence.
Inspired by above results, we observe two factors that control the $\lambda_{\min}(\bm{K}^{(H)})$: (1) $P$, the ``width'' of a DAG; (2) $d_p (p \in [1, P])$, the ``depth'' of a DAG.
However, directly comparing convergences via Eq.~\ref{eq:rule} is non-trivial (see our Appendix~\ref{appendix:dag_lambda_min_conclusion}), because: (1) the complicated form of the NNGP propagation ``$f$'' (Eq.~\ref{lem:propagation}); (2) ``$P$'' and ``$d_p$''s are not truly free variables, as they are still coupled together in the discrete topology space of a fully connected DAG.

Motivated by these two challenges, we propose a practical version of our principle. We simplify the analysis of Eq.~\ref{eq:rule} by reducing ``$P$'' and ``$d_p$''s to only two intuitive variables highly related to the network’s connectivity patterns: the effective depth and effective width.
\begin{definition}[Effective Depth and Width]
    \label{def:depth_width}
    Consider the directed acyclic graph (DAG) of a network. Suppose there are $P$ unique paths connecting the starting and the ending vertex. Denote the number of parameterized operations on the $p$-th path as $d_p, p \in [1, P]$.
    We define:
    \begin{equation}
        \begin{aligned}
            \text{Effective Depth} \quad &\bar{d} = \frac{\sum_{p=1}^P d_p}{P}, \quad\quad
            \text{Effective Width} &\bar{m} = \frac{\sum_{p=1}^P \mathds{1}(d_p > 0)}{\bar{d}}, \\
        \end{aligned}
    \end{equation}
\end{definition}
where $\mathds{1}(d_p > 0) = 1$ if $d_p > 0$ otherwise is $0$. 

\begin{remark}
    \label{remark:op_types}
    In our experiments below, we consider fully-connected or convolutional layers as parameterized operations, ignoring their detailed configurations (kernel sizes, dilations, groups, etc.). We consider skip-connections and pooling layers as non-parameterized operations.
\end{remark}

The effective depth considers the averaging effects of multiple parallel paths, and the effective width considers the amount of unique information flowing from the starting vertex to the end, normalized by the overall depth.
We will demonstrate later in Section~\ref{sec:exp_nasbench_d_m} that performances of networks from diverse architecture spaces show positive correlations with these two aspects.
While $\bar{d}$ and $\bar{m}$ may be loose to predict the best architecture, in Section~\ref{sec:nas} we will show that they are very stable in distinguishing bad ones.
Therefore, using these two principles, we can quickly determine if an architecture is potentially unpromising at almost zero cost, by only analyzing its connectivity without any forward or backward calculations.
\vspace{-0.5em}

\section{Experiments}

\subsection{Empirical convergence confirms our analysis} \label{sec:exp_convergence}

\begin{wrapfigure}{r}{57mm}
\vspace{-7em}
\includegraphics[scale=0.31]{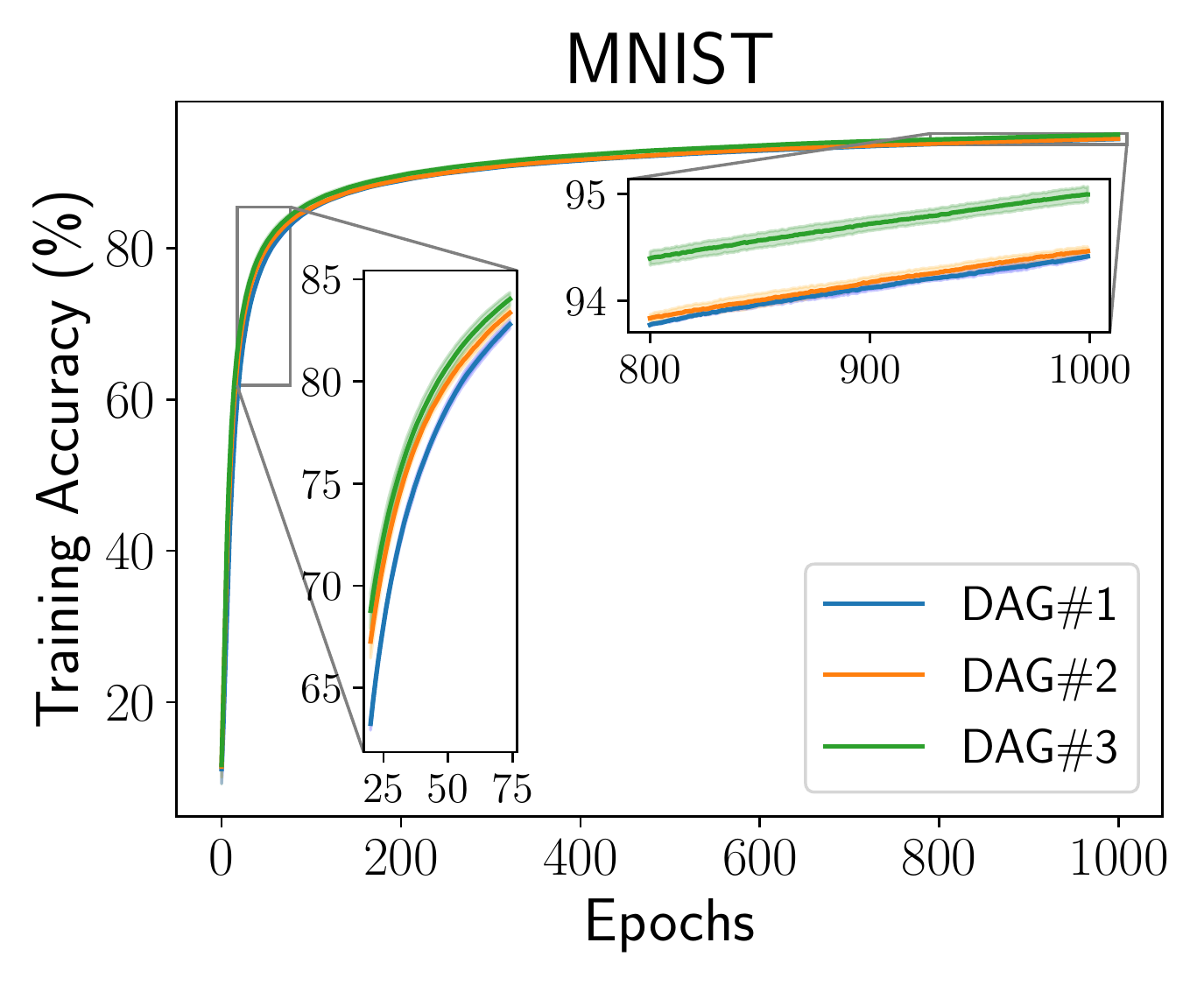}
\includegraphics[scale=0.31]{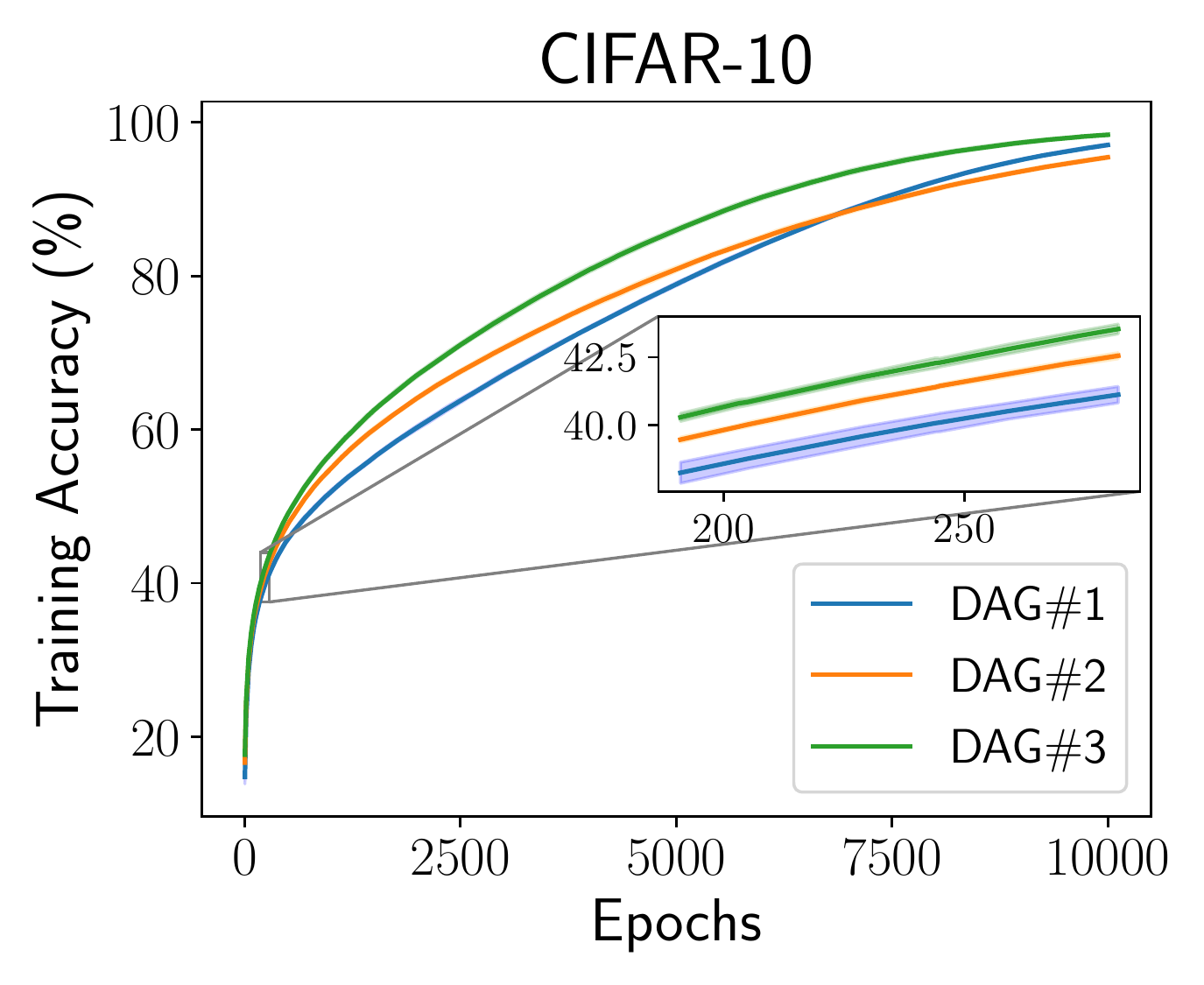}
\vspace{-1em}
\centering
\captionsetup{font=small}
\caption{Empirical convergence of three DAG cases (Figure~\ref{fig:dag_examples}) can verify our theoretical analysis in Section~\ref{sec:convergence}. Small standard deviations of three runs are shown in shadows.}
\label{fig:exp_convergence}
\vspace{-1em}
\end{wrapfigure}

We first experimentally verify our convergence analysis in Section~\ref{sec:propagation_analysis}.
In all cases we use ReLU nonlinearities with Kaiming normal initialization~\cite{he2015delving}.
We build the same three computational graphs of fully-connected layers in Figure~\ref{fig:dag_examples}.
Three networks have hidden layers of a constant width of 1024.
We train the network using SGD with a mini-batch of size 128. The learning rate is fixed at $1\times 10^{-5}$. No augmentation, weight decay, learning rate decay, or momentum is adopted.

Based on our analysis, on both MNIST and CIFAR-10, the convergence rate of DAG\#1 (Figure~\ref{fig:dag_examples} left) is worse than DAG\#2 (Figure~\ref{fig:dag_examples} middle), and is further worse than DAG\#3 (Figure~\ref{fig:dag_examples} right). The experimental observation is consistent with this analysis: the training accuracy of DAG\#3 increases the fastest, and DAG\#2 is faster than DAG\#1.
Although on CIFAR-10 DAG\#1 slightly outperforms DAG\#2 in the end, DAG\#2 still benefits from faster convergence in most early epochs.
This result confirms that by our convergence analysis, we can effectively compare the convergence of different connectivity patterns.

\subsection{Extreme $\bar{d}$ or $\bar{m}$ leads to bad performance on architecture benchmarks} \label{sec:exp_nasbench_d_m}

In this experiment, our goal is to verify the two principles we proposed in Section~\ref{sec:effective_depth_width}. Specifically, we will leverage a large number of network architectures of random connectivities, calculate their $\bar{d}$ and $\bar{m}$, and compare their performance.
We consider popular architecture benchmarks that are widely studied in the community of neural architecture search (NAS). Licenses are publicly available.
\begin{itemize}[leftmargin=*]
    \item The \textit{NAS-Bench-201}~\cite{dong2020bench} provides 15,625 architectures that are stacked by repeated DAGs of four nodes (exactly the same DAG we considered in Section~\ref{sec:convergence} and Figure~\ref{fig:dag_setting}). It contains architecture's performance on three datasets (CIFAR-10, CIFAR-100, ImageNet-16-120 \citep{imagenet16}) evaluated under a unified protocol (i.e. same learning rate, batch size, etc., for all architectures). The operation space contains \textit{zero}, \textit{skip-connection}, \textit{conv}$1\times 1$, \textit{conv}$3\times 3$ \textit{convolution}, and \textit{average pooling} $3\times 3$.
    
    \item Large-scale architecture spaces: NASNet~\cite{zoph2018learning}, Amoeba~\cite{real2019regularized}, PNAS~\cite{liu2018progressive}, ENAS~\cite{pham2018efficient}, DARTS~\cite{liu2018darts}. These spaces have more operation types and more complicated rules on allowed DAG connectivity patterns, see Figure~\ref{fig:teaser} bottom as an example. We refer to their papers for details.
    A benchmark is further proposed in~\cite{radosavovic2019network}: about 5000 architectures are randomly selected from each space above, and in total we have 24827 architectures pretrained on CIFAR-10. We will use this prepared data\footnote{Data is publicly available at {\small \url{https://github.com/facebookresearch/nds}}, only test accuracy available.}.
\end{itemize}

We show our results in Figure~\ref{fig:dag_performance}. From all four plots\footnote{To fairly compare different connectivity patterns, we fix the total number of convolutional layers per model as 3 (out of 6 edges) on NAS-Bench-201, and 5 (out of 10 edges) on the union of large-scale architecture spaces.} (across different architecture spaces and datasets), it can be observed that,
architectures of extreme depths or widths suffer from bad performance.
We also calculate the multi-correlation\footnote{See Appendix~\ref{appendix:corr_multi} for its definition.} ($R$) between the accuracy and the joint of both $\bar{d}$ and $\bar{m}$, and show as legends in Figure~\ref{fig:dag_performance}. All four plots show positive correlations.
It is worth noting that here each dot represents a subset of architectures that share the same $\bar{d}$ and $\bar{m}$, with
possible different fine-grained topology or layer types,
indicating the generalizability of our proposed principles.
The variance of performance in each subset is represented by the radius of the dot.

\begin{remark}
Our notion of ``extreme $\bar{d}$ or $\bar{m}$'' targets a \ul{complete} space of DAGs: given the total number of nodes, any two nodes in the graph can be connected, and there are no topological restrictions or disabled edges (i.e., one cannot introduce any prior to the distribution of graph topologies).
\label{remark:complete_dag}
\end{remark}

\begin{figure}[t!]
\includegraphics[scale=0.22]{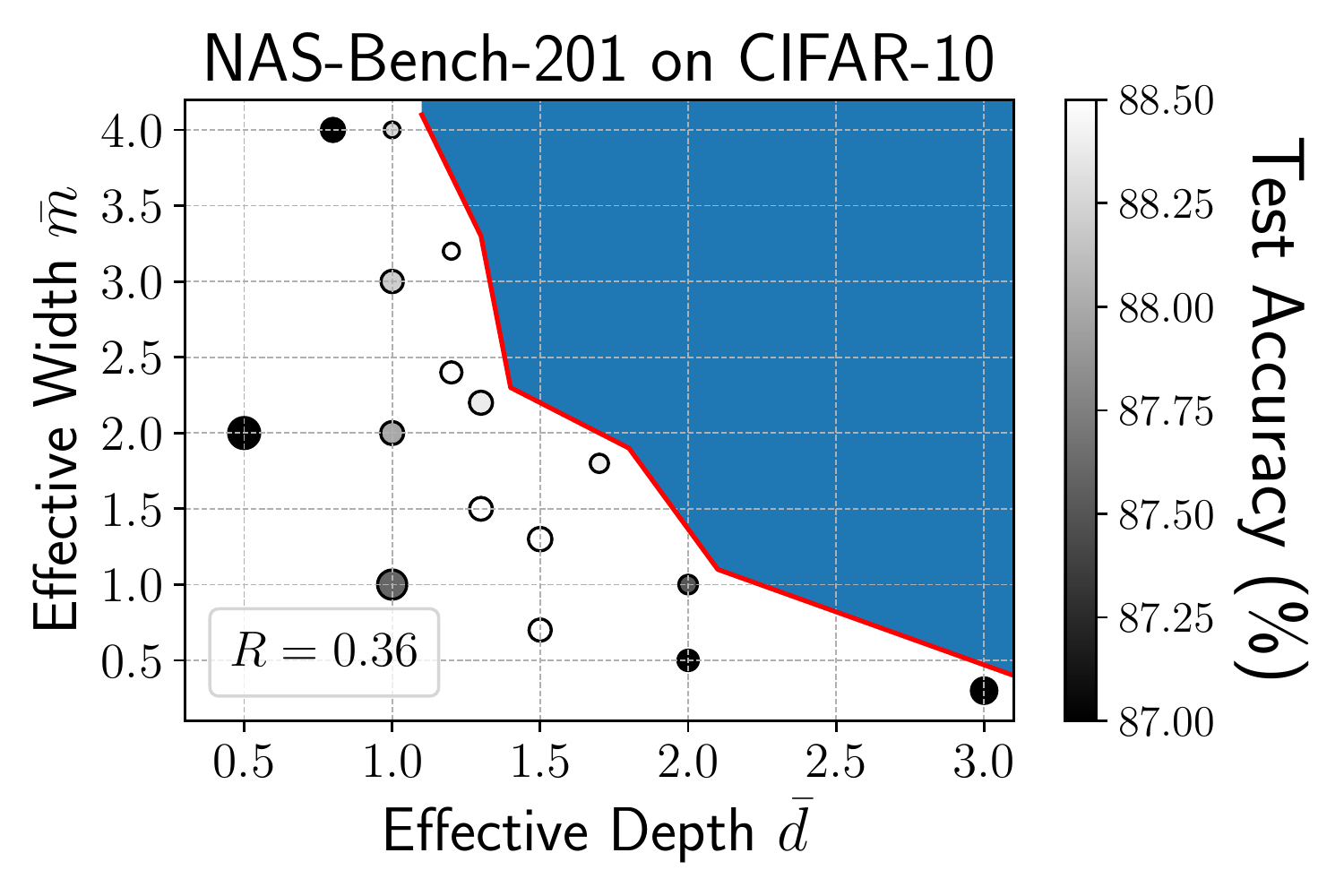}
\includegraphics[scale=0.22]{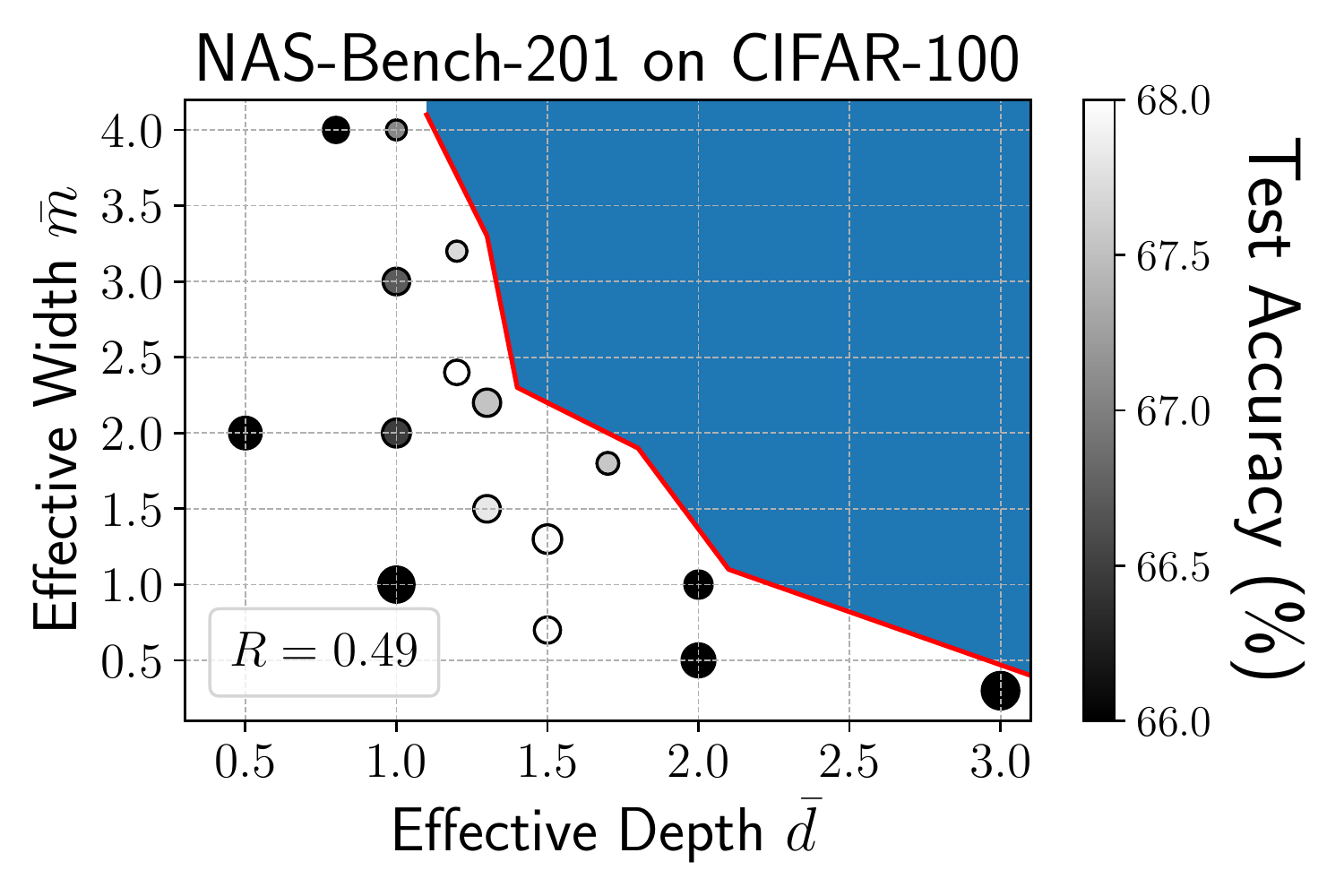}
\includegraphics[scale=0.22]{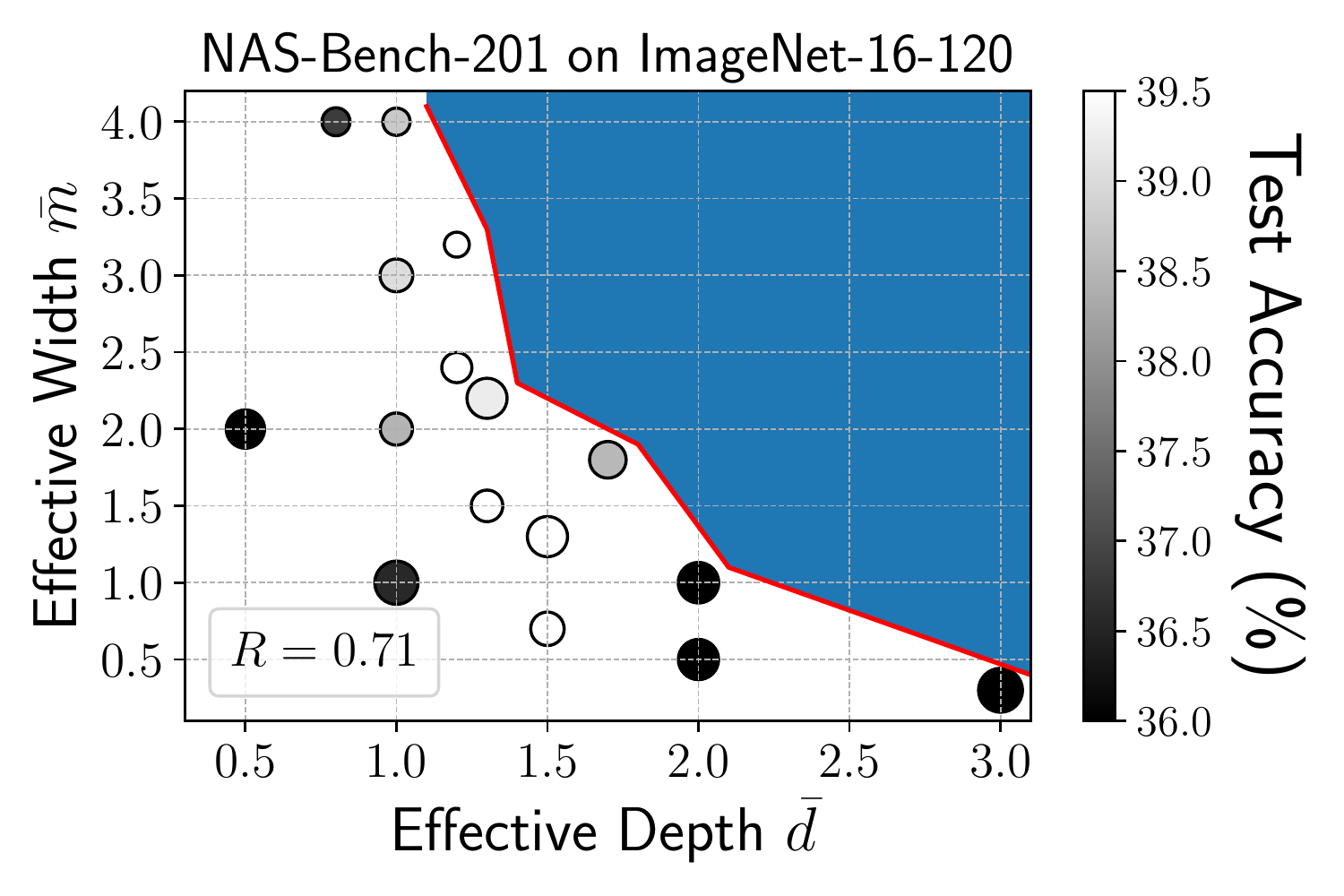}
\includegraphics[scale=0.22]{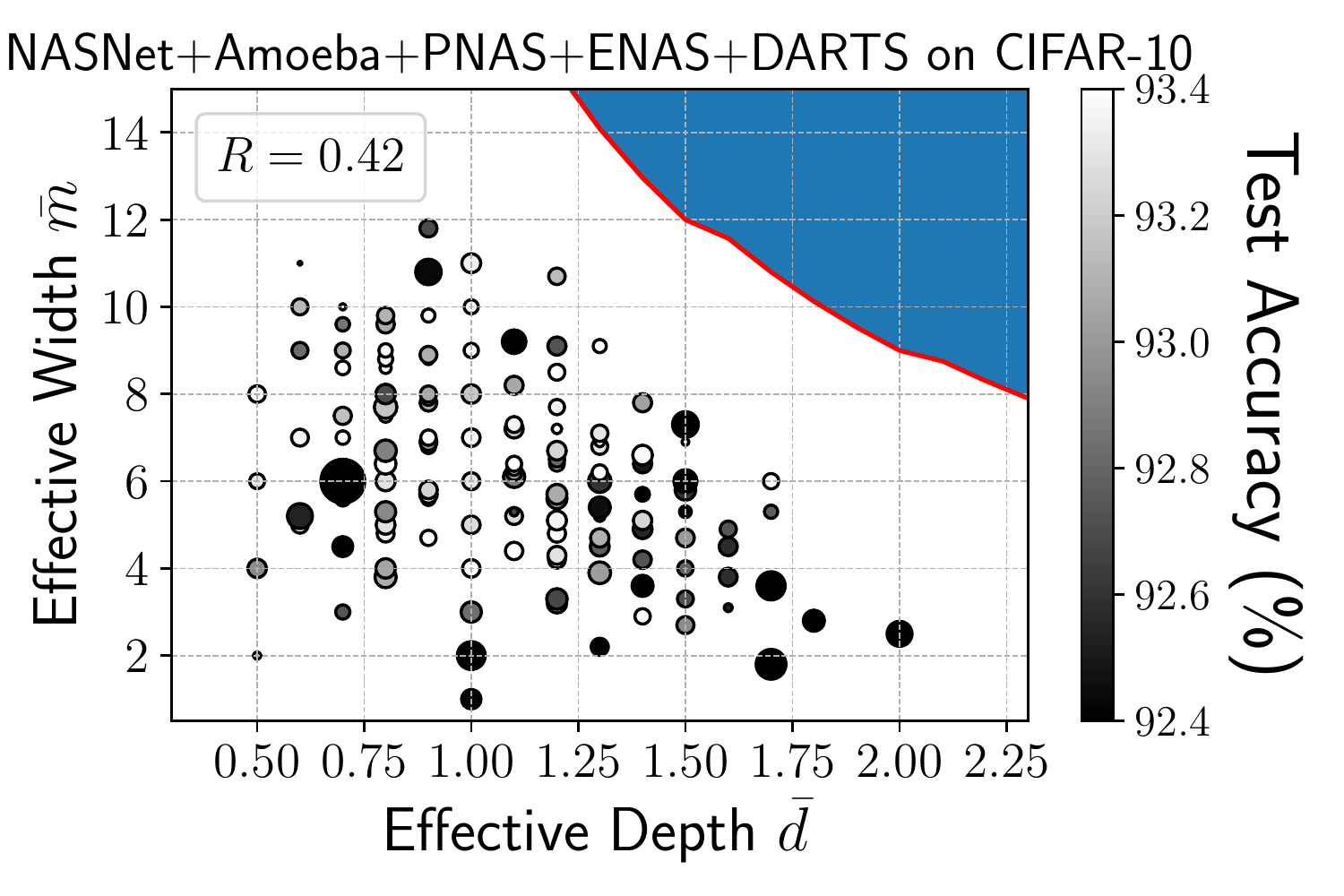}
\vspace{-0.5em}
\centering
\captionsetup{font=small}
\caption{In a complete architecture space, being extremely large or small on either the effective depth ($\bar{d}$) or the effective width ($\bar{m}$) leads a connectivity pattern to bad performance (black) and large variance (large circle size). Each dot represents a subset of architectures of the same $\bar{d}$ and $\bar{m}$. The left three plots are results on CIFAR-10, CIFAR-100, and ImageNet-16-120 from the NAS-Bench-201~\cite{dong2020bench}. The rightmost plot is on CIFAR-10 from a union of multiple large-scale spaces: NASNet~\cite{zoph2018learning}, Amoeba~\cite{real2019regularized}, PNAS~\cite{liu2018progressive}, ENAS~\cite{pham2018efficient}, DARTS~\cite{liu2018darts}. \textcolor{blue}{Blue} areas indicate invalid DAG regions: a DAG cannot be both deep and wide at the same time.}
\label{fig:dag_performance}
\vspace{-1em}
\end{figure}

\subsection{A ``Plug-and-Play'' method for accelerating NAS: bypassing extreme $\bar{d}$ and $\bar{m}$}
\label{sec:nas}

Inspired by the bad performance from extreme $\bar{d}$ or $\bar{m}$ discussed in Section~\ref{sec:exp_nasbench_d_m}, we are motivated to further contribute a ``plug-and-play'' tool to accelerate practical NAS applications.

\paragraph{Method.}
In NAS, the bottleneck of search efficiency is the evaluation cost during the search.
Our core idea is to use $\bar{d}$ and $\bar{m}$ to pre-select architectures that are very likely to be unpromising before we spend extra cost to evaluate them.
Training-based NAS is slow: one has to perform gradient descent on each sampled architecture and use the loss or accuracy to rank different networks.
We instead choose to improve two recent training-free NAS methods: NAS-WOT~\cite{mellor2020neural} and TE-NAS~\cite{chen2020tenas}. This is because our goal here is to accelerate NAS, and these two are state-of-the-art NAS works of extremely low search time cost.
We plan to show that our method can even further accelerate these two training-free NAS works. The rationale is that, our $\bar{d}$ and $\bar{m}$ are even much cheaper: we only do a simple calculation on a network's DAG structure, but NAS-WOT and TE-NAS still have to perform forward or backward on a mini-batch of real data.
Although our $\bar{d}$ and $\bar{m}$ are only inspired from the optimization perspective, not the complexity or generalization, our method mainly filters out bad architectures at a coarse level, but does not promote elites in a fine-grained way.
\vspace{-0.5em}
\begin{enumerate}[align=left]
    \item[\textit{NAS-WOT}] proposed to rank architectures by their local linear maps. Given a mini-batch of data, each point can be identified by its activation pattern through the network, and a network that can assign more unique activation patterns will be considered promising. NAS-WOT uses random sampling to search the architectures and rank them based on this training-free score.\\
    \textbf{Our version}: We will skip potentially unpromising architecture before leaving them for NAS-WOT to calculate their scores.
    \item[\textit{TE-NAS}] proposed to rank architectures by combining both the trainability (condition number of NTK) and expressivity (number of linear regions). TE-NAS progressively pruned a super network to a single-path network, by removing unpromising edges of low scores.\\
    \textbf{Our version}:
    We will skip potentially unpromising architecture before leaving them for TE-NAS to prune.
\end{enumerate}
\vspace{-0.5em}
We provide a pseudocode algorithm in Appendix~\ref{appendix:algorithm} to demonstrate the usage of our method.

\paragraph{How to choose $\bar{d}$ and $\bar{m}$?}
Although $\bar{d}$ and $\bar{m}$ are hyperparameters, it is worth noting that they can be determined in a highly principled way.
In practice,
given a search space,
we only calculate the center $(\bar{d}^*, \bar{m}^*) = (\frac{\bar{d}_{\max}+\bar{d}_{\min}}{2}, \frac{\bar{m}_{\max}+\bar{m}_{\min}}{2})$ and the radius $(r_{\bar{d}}, r_{\bar{m}}) = (\frac{\bar{d}_{\max}-\bar{d}_{\min}}{2}, \frac{\bar{m}_{\max}-\bar{m}_{\min}}{2})$.
We by default only keep architectures within half of the radius from the center for evaluation:
$|\bar{d} - \bar{d}^*| \leq \frac{1}{2} r_{\bar{d}}$ and $|\bar{m} - \bar{m}^*| \leq \frac{1}{2} r_{\bar{m}}$.
This principle leads to $(\bar{d}^*, \bar{m}^*) = (1.5, 10.5)$ with $\frac{1}{2}(r_{\bar{d}}, r_{\bar{m}}) = (0.7, 4.8)$ on DARTS, and $(\bar{d}^*, \bar{m}^*) = (1.6, 2.2)$ with $\frac{1}{2}(r_{\bar{d}}, r_{\bar{m}}) = (0.7, 0.9)$ on NAS-Bench-201.
\vspace{-0.5em}

\paragraph{Implementation settings.} We train searched architectures for 250 epochs using SGD, with a learning rate as 0.5, a cosine scheduler, momentum as 0.9, weight decay as $3\times 10^{-5}$, and a batch size as 768.
This setting follows previous works~\cite{abbdelfattah2020zero,mellor2020neural,zhang2021neural,lopes2021guided,ye2022beta,hou2021single,chen2020drnas,chen2020tenas}.
\vspace{-0.5em}

\begin{table*}[!t]
    \centering
    \caption{NAS Search Performance in DARTS space on ImageNet. Our standard deviations over three random runs are included in parentheses.}
    \scriptsize
    \resizebox{0.9\textwidth}{!}{
    \begin{threeparttable}
    \begin{tabular}{lccccc}
    \toprule
    
    \multirow{2}*{\textbf{Architecture}} & \multicolumn{2}{c}{\textbf{Test Error(\%)}} & \multirow{2}*{\textbf{Params. (M)}} &
    \multirow{2}*{\textbf{Search Cost (GPU days)}} & 
    \multirow{2}*{\textbf{Search Method}} \\ \cline{2-3}
    
    & top-1 & top-5 & & & \\ \midrule
    
    NASNet-A \cite{zoph2018learning} & 26.0 & 8.4 & 5.3 & 2000 & RL \\ 
    AmoebaNet-C \cite{real2019regularized} & 24.3 & 7.6 & 6.4 & 3150 & evolution \\ 
    PNAS \cite{liu2018progressive} & 25.8 & 8.1 & 5.1 & 225 & SMBO \\ 
    MnasNet-92 \cite{tan2019mnasnet}\tnote{$\dagger$} & 25.2 & 8.0 & 4.4 &  288 (TPU) & RL \\ 
    \midrule
    DARTS (2nd) \cite{liu2018darts} & 26.7 & 8.7 & 4.7 & 4.0\tnote{$\ddagger$} & gradient \\
    SNAS (mild) \cite{xie2018snas} & 27.3 & 9.2 & 4.3 & 1.5 & gradient \\ 
    GDAS \cite{dong2019searching} & 26.0 & 8.5 & 5.3 & 0.21 & gradient \\ 
    BayesNAS \cite{zhou2019bayesnas} & 26.5 & 8.9 & 3.9 & 0.2 & gradient \\ 
    P-DARTS \cite{chen2019progressive} & 24.4 & 7.4 & 4.9 & 0.3 & gradient \\ 
    PC-DARTS \cite{xu2019pc} & 25.1 & 7.8 & 5.3 & 0.1\tnote{$\ddagger$} & gradient \\
    PC-DARTS (ImageNet) \cite{xu2019pc}\tnote{$\dagger$} & 24.2 & 7.3 & 5.3 & 3.8 & gradient \\
    ProxylessNAS (GPU) \cite{cai2018proxylessnas}\tnote{$\dagger$} & 24.9 & 7.5 & 7.1 & 8.3 & gradient \\ 
    SGAS (Cri 1. avg) \cite{li2020sgas} & 24.42 (0.16) & 7.29 (0.09) & 5.3 & 0.25 & gradient \\
    DrNAS \cite{chen2020drnas}\tnote{$\dagger$} & 23.7 & 7.1 & 5.7 & 4.6 & gradient \\
    \midrule
    NAS-WOT~\cite{mellor2020neural}\tnote{$\dagger$}~~\tnote{$\star$} & 26.2 & 8.2 & 4.4 & 0.0036\tnote{$\ddagger$} & training-free \\ 
    NAS-WOT \textbf{+ DAG (ours)}\tnote{$\dagger$} & 25.9 (0.4) & 8.2 & 4.4 & 0.003\tnote{$\ddagger$} & training-free \\
    TE-NAS \cite{chen2020tenas}\tnote{$\dagger$} & 24.5 & 7.5 & 5.4 & 0.17\tnote{$\ddagger$} & training-free \\
    TE-NAS \textbf{+ DAG (ours)}\tnote{$\dagger$} & 24.2 (0.3) & 7.4 & 6.1 & 0.1\tnote{$\ddagger$} & training-free \\
    \bottomrule
    \end{tabular}
    \begin{tablenotes}
        \item[$\dagger$] Architectures searched on ImageNet. Other methods searched on CIFAR-10 and then transfered to the ImageNet.
        \item[$\star$] Our reproduced result, the original work did not provide results on the ImageNet.
        \item[$\ddagger$] Recorded on a single GTX 1080Ti GPU.
    \end{tablenotes}
    \end{threeparttable}}
    \label{tab:imagenet}
\vspace{-1.5em}
\end{table*}

\begin{table}[t!]
    \centering
    \caption{Search Performance from NAS-Bench-201. ``optimal'' indicates the best test accuracy achievable in the space. Our standard deviations over three random runs are included in parentheses.}
    \resizebox{0.8\textwidth}{!}{
    \begin{tabular}{lcccccc}
    \toprule
    \textbf{Architecture} & \textbf{CIFAR-10} & \textbf{CIFAR-100} & \textbf{ImageNet-16-120} & \textbf{\tabincell{c}{Search Cost\\(GPU sec.)}} & \textbf{\tabincell{c}{Search\\Method}} \\ \midrule
    ResNet \cite{he2016deep} & 93.97 & 70.86 & 43.63 & - & - \\ \midrule
    RSPS \cite{li2020random} & $87.66(1.69)$ & $58.33(4.34)$ & $31.14(3.88)$ & 8007.13 & random \\
    ENAS \cite{pham2018efficient} & $54.30(0.00)$ & $15.61(0.00)$ & $16.32(0.00)$ & 13314.51 & RL \\
    DARTS (1st) \cite{liu2018darts} & $54.30(0.00)$ & $15.61(0.00)$ & $16.32(0.00)$ & 10889.87 & gradient \\
    DARTS (2nd) \cite{liu2018darts} & $54.30(0.00)$ & $15.61(0.00)$ & $16.32(0.00)$ & 29901.67 & gradient \\
    GDAS \cite{dong2019searching} & $93.61(0.09)$ & $70.70(0.30)$ & $41.84 (0.90)$ & 28925.91 & gradient \\
    \midrule
    WOT~\cite{mellor2020neural} & 92.81 (0.99) &  69.48 (1.70) & 43.10 (3.16) & 30.01 & training-free\\
    WOT + DAG (ours) & 92.98 (0.78) & 69.86 (1.24) & 43.44 (2.64) & 17.95 (-40.2\%) & training-free\\
    TE-NAS~\cite{chen2020tenas} & 93.9 (0.47) & 71.24 (0.56) & 42.38 (0.46) & 1558 & training-free\\
    TE-NAS + DAG (ours) & 93.6 (0.37) & 71.26 (0.77) & 44.38 (0.76) & 1077 (-30.9\%) & training-free\\
    FP-NAS~\cite{fpnas} & 77.4 (16.6) & 64.7 (5.3) & 26.7 (10.4) & 4837 & gradient\\
    FP-NAS + DAG (ours) & 93.3 (0.3) & 70.8 (0.4) & 44.5 (1.4) & 2612 (-46.0\%) & gradient \\
    \midrule
    \textbf{Optimal} & 94.37 & 73.51 & 47.31 & - & - \\
    \bottomrule
    \end{tabular}\label{table:nasbench201}
    }
\vspace{-1.5em}
\end{table}

\paragraph{DARTS space on ImageNet.}
As shown in Table~\ref{tab:imagenet}, for both two training-free NAS methods, by adopting our pre-filtration strategy, we can further reduce the search time cost and achieve better search results. For NAS-WOT, we can save 16.7\% search time cost and reduce 0.3\% top-1 error. For TE-NAS, we can significantly save 41.2\% search time cost, and still improve the top-1 error by 0.4\%.
FLOPs of our search architectures are 0.68G (TE-NAS + DAG) and 0.56G (WOT + DAG).

\begin{wraptable}{r}{65mm}
\vspace{-0.5em}
    \centering
    \captionsetup{font=small}
    \caption{Our method is robust to choices of $\bar{d}$ and $\bar{m}$ (WOT~\cite{mellor2020neural} on NAS-Bench-201).}
    \vspace{-0.5em}
    \scriptsize
    \resizebox{0.45\textwidth}{!}{
    {\renewcommand{\arraystretch}{1.4}
    \begin{tabular}{cc}
        \toprule
        Ranges of accepted $\bar{d}$ and $\bar{m}$ & CIFAR-100 \\ \midrule
        $|\bar{d} - \bar{d}^*| \leq r_{\bar{d}}$, $|\bar{m} - \bar{m}^*| \leq r_{\bar{m}}$ (baseline) & 69.48 (1.70) \\ \midrule
        $|\bar{d} - \bar{d}^*| \leq \frac{3}{4} r_{\bar{d}}$, $|\bar{m} - \bar{m}^*| \leq \frac{3}{4} r_{\bar{m}}$ & 69.51 (1.70) \\
        $|\bar{d} - \bar{d}^*| \leq \frac{1}{2} r_{\bar{d}}$, $|\bar{m} - \bar{m}^*| \leq \frac{1}{2} r_{\bar{m}}$ & 69.86 (1.24) \\
        $|\bar{d} - \bar{d}^*| \leq \frac{1}{4} r_{\bar{d}}$, $|\bar{m} - \bar{m}^*| \leq \frac{1}{4} r_{\bar{m}}$ & 69.97 (1.19) \\ \bottomrule
    \end{tabular}}}
    \label{tab:201_dm_ranges}
\end{wraptable}

\paragraph{NAS-Bench-201 Space.}
As shown in Table~\ref{table:nasbench201}, for both NAS-WOT and TE-NAS, we reduce over 30\% search time cost with strong accuracy.
We also include a training-based method FP-NAS~\cite{fpnas}, where we even achieve 46\% seach cost reduction with better performance.
Moreover, we show that our method is robust to the choices of $\bar{d}$ and $\bar{m}$. In the ablation study in Table~\ref{tab:201_dm_ranges}, by changing different ranges of $\bar{d}$ and $\bar{m}$, our method remains strong over the WOT baseline.

\section{Conclusion}

In this work, we show that it is possible to conduct fine-grained convergence analysis on networks of complex connectivity patterns. By analyzing how an NNGP kernel propagates through the networks, we can fairly compare different networks' bounds of convergence rates. This theoretical analysis and comparison are empirically verified on MNIST and CIFAR-10. To make our convergence analysis more practical and general, we propose two intuitive principles on how to design a network's connectivity patterns: the effective depth and the effective width. Experiments on diverse architecture benchmarks and datasets demonstrate that networks with an extreme depth or width show bad performance, indicating that both the depth and width are important. Finally, we apply our principles to the large-scale neural architecture search application, and our method can largely accelerate the search cost of two training-free efficient NAS works with faster and better search performance. Our work bridge the gap between the Deep Learning theory and the application part, making the theoretical analysis more practical in architecture designs.

\section*{Acknowledgement}

B. Hanin and Z. Wang are supported by NSF Scale-MoDL (award numbers: 2133806, 2133861).

{
\small
 \bibliographystyle{plain}
\bibliography{neurips_2022}
}

\section*{Checklist}

The checklist follows the references.  Please
read the checklist guidelines carefully for information on how to answer these
questions.  For each question, change the default \answerTODO{} to \answerYes{},
\answerNo{}, or \answerNA{}.  You are strongly encouraged to include a {\bf
justification to your answer}, either by referencing the appropriate section of
your paper or providing a brief inline description.  For example:
\begin{itemize}
  \item Did you include the license to the code and datasets? \answerYes{See Section~\ref{gen_inst}.}
  \item Did you include the license to the code and datasets? \answerNo{The code and the data are proprietary.}
  \item Did you include the license to the code and datasets? \answerNA{}
\end{itemize}
Please do not modify the questions and only use the provided macros for your
answers.  Note that the Checklist section does not count towards the page
limit.  In your paper, please delete this instructions block and only keep the
Checklist section heading above along with the questions/answers below.

\begin{enumerate}

\item For all authors...
\begin{enumerate}
  \item Do the main claims made in the abstract and introduction accurately reflect the paper's contributions and scope?
    \answerYes{}
  \item Did you describe the limitations of your work?
    \answerYes{In Remark~\ref{remark:complete_dag}, our claim ``Extreme $\bar{d}$ or $\bar{m}$ leads to bad performance'' requires architectures from a complete space of DAGs, without introducing any prior on the DAG topology.}
  \item Did you discuss any potential negative societal impacts of your work?
    \answerNo{Our work does not have negative societal impacts}
  \item Have you read the ethics review guidelines and ensured that your paper conforms to them?
    \answerYes{}
\end{enumerate}

\item If you are including theoretical results...
\begin{enumerate}
  \item Did you state the full set of assumptions of all theoretical results?
    \answerYes{}
        \item Did you include complete proofs of all theoretical results?
    \answerYes{We include our proofs in supplement.}
\end{enumerate}

\item If you ran experiments...
\begin{enumerate}
  \item Did you include the code, data, and instructions needed to reproduce the main experimental results (either in the supplemental material or as a URL)?
    \answerYes{We include our code and Readme of instructions in supplement. Data is publicly available online.}
  \item Did you specify all the training details (e.g., data splits, hyperparameters, how they were chosen)?
    \answerYes{See Section~\ref{sec:nas}.}
        \item Did you report error bars (e.g., with respect to the random seed after running experiments multiple times)?
    \answerYes{We run experiments for three random seeds.}
        \item Did you include the total amount of compute and the type of resources used (e.g., type of GPUs, internal cluster, or cloud provider)?
    \answerYes{See Table~\ref{tab:imagenet}.}
\end{enumerate}

\item If you are using existing assets (e.g., code, data, models) or curating/releasing new assets...
\begin{enumerate}
  \item If your work uses existing assets, did you cite the creators?
    \answerYes{}
  \item Did you mention the license of the assets?
    \answerYes{Licenses are publicly available}
  \item Did you include any new assets either in the supplemental material or as a URL?
    \answerNA{}
  \item Did you discuss whether and how consent was obtained from people whose data you're using/curating?
    \answerNo{}
  \item Did you discuss whether the data you are using/curating contains personally identifiable information or offensive content?
    \answerNA{}
\end{enumerate}

\item If you used crowdsourcing or conducted research with human subjects...
\begin{enumerate}
  \item Did you include the full text of instructions given to participants and screenshots, if applicable?
    \answerNA{}
  \item Did you describe any potential participant risks, with links to Institutional Review Board (IRB) approvals, if applicable?
    \answerNA{}
  \item Did you include the estimated hourly wage paid to participants and the total amount spent on participant compensation?
    \answerNA{}
\end{enumerate}

\end{enumerate}


\newpage

\appendix

\section{Pseudocode of our ``Plug-and-Play'' method}
\label{appendix:algorithm}

Here we provide a pseudocode of using our ``Plug-and-Play'' method in the simplified WOT algorithm~\cite{mellor2020neural}.
Black lines are from original WOT.
Red lines are our method, which is extremely easy to use and cheap to calculate.

\begin{algorithm2e}[h!]
\caption{Pseudocode of ``WOT + ours'' in a PyTorch-like style.}
\label{alg:code}
\algcomment{\fontsize{7.2pt}{0em}\selectfont \texttt{bmm}: batch matrix multiplication; \texttt{mm}: matrix multiplication; \texttt{cat}: concatenation.
}
\definecolor{codeblue}{rgb}{0.25,0.5,0.5}
\lstset{
  backgroundcolor=\color{white},
  basicstyle=\fontsize{7.2pt}{7.2pt}\ttfamily\selectfont,
  columns=fullflexible,
  breaklines=true,
  captionpos=b,
  commentstyle=\fontsize{7.2pt}{7.2pt}\color{codeblue},
  keywordstyle=\fontsize{7.2pt}{7.2pt},
  escapechar={|},
}
\begin{lstlisting}[language=python]
# depth_center, r_depth: center and radius of depths defined in Section|~\ref{sec:nas}|
# width_center, r_width: center and radius of widths defined in Section|~\ref{sec:nas}|
# search_space: a given space of architectures, like DARTS or NAS-Bench-201
# N: number of sampled architectures to evaluate
# WOT: the Jacobian-based scoring method in|~\cite{mellor2020neural}|.

archs = random_sample(search_space, N) # randomly sample N architectures to evaluate
scores = []
for arch in archs:
    |\color{red}depth, width = get\_depth\_width(arch)| # get DAG's effective depth and width
    |\color{red}if not (abs(depth - depth\_center) <= 0.5 * r\_depth and abs(width - width\_center) <= 0.5 * r\_width):|
        # this DAG has large depth or large width, bypass it
        |\color{red}scores.append(-inf)|
        |\color{red}continue|
    score = WOT(arch)
    scores.append(score)
return archs[argmax(scores)]
\end{lstlisting}
\end{algorithm2e}

\section{Proofs for Section~\ref{sec:propagation_analysis}}

\begin{customlemma}{\ref{lem:propagation}}[Propagation of $\bm{K}$]
Define the propagation as $\bm{K}^{(l)} = f(\bm{K}^{(l-1)})$ and $\bm{b}^{(l)} =  g(\bm{b}^{(l-1)})$. When the edge operation is a linear transformation, we have,
{\small
\begin{equation}
\begin{aligned}
   \bm{K}^{(l)}_{ii} & = f(\bm{K}^{(l-1)}_{ii}) = \int  \mathcal{D}_z c_\sigma \sigma^2 \bigg(\sqrt{\bm{K}_{ii}^{(l-1)}}z \bigg)  \\ 
 \bm{K}^{(l)}_{ij} & = f(\bm{K}^{(l-1)}_{ij}) = \int \mathcal{D}_{z_1} \mathcal{D}_{z_2} c_\sigma  \sigma\bigg( \sqrt{\bm{K}^{(l-1)}_{ii} } z_1 \bigg) \sigma \bigg(\sqrt{\bm{K}^{(l-1)}_{jj}}( \bm{C}_{ij}^{(l-1)} z_1 + \sqrt{1-(\bm{C}_{ij}^{(l-1)})^2} z_2) \bigg)   \\
\bm{C}_{ij}^{(l)} & = \bm{K}_{ij}^{(l)}/\sqrt{ \bm{K}_{ii}^{(l)} \bm{K}_{jj}^{(l)}}\\ 
\bm{b}^{(l)}_{i} & = g(\bm{b}^{(l-1)}_{i}) = \int  \mathcal{D}_{z} \sqrt{c_\sigma} \sigma \big(\sqrt{ \bm{K}_{ii}^{(l)}} z \big)
\end{aligned}
\end{equation}
}
where $z$, $z_1$ and $z_2$ are independent
standard Gaussian random variables. Besides,
$\int \mathcal{D}_z = \frac{1}{\sqrt{2\pi}} \int dz e^{-\frac{1}{2}z^2}$ is the measure for a normal distribution.
\end{customlemma}

\begin{proof}[Proof of Lemma~\ref{lem:propagation}]

According to the property of Gaussian process for neural networks \cite{lee2017deep}, for each layer $l$, we have
$$
\begin{aligned}
\bm{K}_{ii}^{(l)} & =  \int \mathcal{D}_z  c_\sigma  \sigma^2\bigg(\sqrt{\bm{K}_{ii}^{(l-1)}}z \bigg) \\
\bm{K}_{ij}^{(l)} & = \int  \mathcal{D}_{z_1} \mathcal{D}_{z_2} c_\sigma \sigma(u) \sigma(v) 
\end{aligned}
$$
where $u = \sqrt{\bm{K}_{ii}^{(l-1)}} z_1 $ and $v =\sqrt{\bm{K}_{jj}^{(l-1)}} \bigg({\bm{C}_{ij}^{(l-1)}} z_1 + \sqrt{1-(\bm{C}_{ij}^{(l-1)})^2} z_2 \bigg) $, with $\bm{C}_{ij}^{(l)}  = \bm{K}_{ij}^{(l)}/\sqrt{ \bm{K}_{ii}^{(l)} \bm{K}_{jj}^{(l)}}$. Besides,
$\int \mathcal{D}_z = \frac{1}{\sqrt{2\pi}} \int dz e^{-\frac{1}{2}z^2}$ is the measure for a normal distribution. 

Note that the diagonal entry $\bm{K}_{ii}^{(l)}$ of NNGP is a self-correlation term thus governed by one random variance. On the other hand, $\bm{K}_{ij}^{(l)}$ is the correlation between $\bm{X}_{i}^{(l)} $ and $\bm{X}_{j}^{(l)}$, therefore it contains two independent random variables. Finally, we calculate the bias term:
\begin{align*}
  \bm{b}^{(l)}_{i} & = g(\bm{b}^{(l-1)}_{i}) =   \sqrt{c_\sigma} \int \mathcal{D}_{z} \sigma(\sqrt{ \bm{K}_{ii}^{(l)}} z)
\end{align*}

\end{proof}

\begin{customlemma}{\ref{lem:lambda_min_2x2}}
For a positive definite symmetric matrix $\bm{K} \in \mathbb{R}^{N \times N}$, the smallest eigenvalue is bounded by the smallest eigenvalue of its $2 \times 2$ principal sub-matrix.
\begin{equation}
\lambda_{\min}(\bm{K}) \le \min_{i \neq j} \lambda_{\min} \begin{bmatrix}
     \bm{K}_{ii}  & \bm{K}_{ij}\\
     \bm{K}_{ji} & \bm{K}_{jj} 
 \end{bmatrix}
\end{equation}
\end{customlemma}

\begin{proof}[Proof of Lemma~\ref{lem:lambda_min_2x2}]

We first have the \textbf{Cauchy Interlace Theorem}:
Suppose $\bm{A} \in \mathbb{R}^{n\times n}$ is symmetric. Let $\bm{A}^{(n-1)} \in \mathbb{R}^{(n-1)\times(n-1)}$ be a principal submatrix of $\bm{A}$ (obtained by removing both $i$-th row and $i$-th column for any $i \in [1, n]$). Suppose $\bm{A}$ has eigenvalues $\lambda_1^{(n)} \geq \lambda_2^{(n)} \geq \cdots \geq \lambda_n^{(n)}$ and $\bm{A}^{(n-1)}$ has eigenvalues $\lambda_1^{(n-1)} \geq \lambda_2^{(n-1)} \geq \cdots \geq \lambda_{n-1}^{(n-1)}$. Then:
$$
\lambda_1^{(n)} \geq \lambda_1^{(n-1)} \geq \lambda_2^{(n)} \geq \lambda_2^{(n-1)} \geq \cdots \geq \lambda_{n-1}^{(n-1)} \geq \lambda_n^{(n)}.
$$
By repeating this theorem for $n-2$ times, we can have:
$$
\lambda_2^{(2)} \geq \lambda_3^{(3)} \geq \cdots \geq \lambda_{n-1}^{(n-1)} \geq \lambda_n^{(n)}.
$$
This means we can upper bound $\lambda_n^{(n)}$ by the minimal smallest eigenvalue of any $2\times 2$ principal submatrix of $\bm{A}$.

\end{proof}

\section{Detailed Analysis on three DAGs in Figure~\ref{fig:dag_examples} and Section~\ref{sec:propagation_analysis}.}
\label{appendix:dag_lambda_min_conclusion}

We take the condition of ReLU that $\sigma(x) = \max \{0,x \}$ and $c_\sigma = 2$ into Lemma \ref{lem:propagation} and obtain:
$$
\begin{aligned}
 \bm{K}^{(l)}_{ii} & = \int c_\sigma \mathcal{D}_z \sigma^2 \bigg(\sqrt{\bm{K}_{ii}^{(l-1)}}z \bigg) =\bm{K}^{(l-1)}_{ii} \\
 \bm{K}^{(l)}_{ij}  
 & =\frac{2\bm{C}^{(l-1)}_{ij} \arcsin{\bm{C}^{(l-1)}_{ij}}+2\sqrt{1-(\bm{C}^{(l-1)}_{ij})^2}+ \pi \bm{C}^{(l-1)}_{ij}}{2 \pi}  \cdot  \sqrt{\bm{K}^{(l)}_{ii} \bm{K}^{(l)}_{jj} }\\
\bm{b}^l_{i} & = \int \sqrt{c_\sigma \mathcal{D}_z} \sigma(\sqrt{\bm{K}_{ii}^{(l-1)}}z) =  \frac{\sqrt{2}}{2} .
\end{aligned}
$$

We provide more properties regarding the propagation of normalized non-diagonal elements in NNGP by studying the function
$$
\bm{C}^{(l)}_{ij}  
  = h(\bm{C}^{(l-1)}_{ij}) = \frac{2\bm{C}^{(l-1)}_{ij} \arcsin{\bm{C}^{(l-1)}_{ij}}+2\sqrt{1-(\bm{C}^{(l-1)}_{ij})^2}+ \pi \bm{C}^{(l-1)}_{ij}}{2 \pi}.
$$
It is known that $f$ is differentiable and satisfies:
$$
h'(\bm{C}^{(l-1)}_{ij}) =  \frac{1}{\pi} \arcsin(\bm{C}^{(l-1)}_{ij}) + \frac{1}{2}; ~ h''(\bm{C}^{(l-1)}_{ij}) =  \frac{1}{\pi \sqrt{1-(\bm{C}^{(l-1)}_{ij})^2 }}.
$$
Therefor we have $h(\bm{C}^{(l-1)}_{ii}) > \bm{C}^{(l-1)}_{ii}$, for $\bm{C}^{(l-1)}_{ii} \in (0,1)$. Finally, we show the Taylor expansion for $h(\bm{C}^{(l-1)}_{ii})$ when $\bm{C}^{(l-1)}_{ii}$ is close to 1:
$$
h(\bm{C}^{(l)}_{ij}) \overset{\bm{C}^{(l)}_{ij} \rightarrow 1-}{=}  \bm{C}^{(l)}_{ij}+ \frac{2\sqrt{2}}{3 \pi}(1- \bm{C}^{(l)}_{ij})^{3/2} + O((1-\bm{C}^{(l)}_{ij})^{5/2}) .
$$
This can be obtained by expanding each term in $h(\bm{C}_{ij}^{(l-1)})$.

Before our analysis, we re-state some assumptions and facts:
\begin{itemize}[leftmargin=*]
    \item $h(\cdot)$ is a monotonically increasing function in $[0, 1)$, and         $\lim_{\bm{C}_{ij}^{(l-1)} \rightarrow 1-} h(\bm{C}_{ij}^{(l-1)}) = 1$. Please refer to the blue curve in Figure 2(b) in~\cite{hayou2019impact} for the hehavior of $h(\cdot)$.
    \item We assume no parallel data points in the training data, thus $\bm{C}_{ij}^{(0)} \in [0, 1)$ and $\bm{K}_{ij}^0 \in [0, 1)$. This further indicates $h(\bm{C}_{ij}^{(0)}), h(\bm{K}_{ij}^{(0)}) \in [1/\pi, 1)$.
    \item $f(\bm{K}_{ij}^{(l-1)}) > \bm{K}_{ij}^{(l-1)}$.
\end{itemize}

We first show that $\lambda_{\text{dag1}} < \lambda_{\text{dag2}}$, namely, the upper bounds of the least eigenvalues of DAG\#1 is smaller than that of DAG\#2 in Figure~\ref{fig:dag_examples}:
$$
\begin{aligned}
\lambda_{\text{dag2}} &= 3(1 - f(\bm{K}^{(0)}_{ij})) > 1 - f(\bm{K}^{(0)}_{ij}) > 1 - f^3(\bm{K}^{(0)}_{ij}) = \lambda_{\text{dag1}}.
\end{aligned}
$$

Before we compare $\lambda_{\text{dag2}}$ and $\lambda_{\text{dag3}}$, we first try to simplify $\lambda_{\text{dag3}}$. Specifically, we have:
$$
\begin{aligned}
f(\bm{K}^{(0)}_{ij}) &= h(\bm{C}^{(0)}_{ij})\sqrt{\bm{K}^{(1)}_{ii}\bm{K}^{(1)}_{jj}} = h(\bm{C}^{(0)}_{ij})
\end{aligned}
$$
and
$$
\begin{aligned}
f(2f(\bm{K}^{(0)}_{ij})) &= h(\frac{2f(\bm{K}^{(0)}_{ij})}{\sqrt{2f(\bm{K}^{(0)}_{ii}) \cdot 2f(\bm{K}^{(0)}_{jj})}})\sqrt{2\bm{K}^{(1)}_{ii} \cdot 2\bm{K}^{(1)}_{jj}} \\
&= h(f(\bm{K}^{(0)}_{ij})) \cdot 2= 2h(h(\bm{C}^{(0)}_{ij}))
\end{aligned}
$$
Now we compare $\lambda_{\text{dag2}}$ and $\lambda_{\text{dag3}}$:
$$
\begin{aligned}
\lambda_{\text{dag3}} - \lambda_{\text{dag2}} &= 4 - \left( \bm{K}^{(0)}_{ij} + f(\bm{K}^{(0)}_{ij}) + f(2f(\bm{K}^{(0)}_{ij})) \right) - 3(1 - f(\bm{K}^{(0)}_{ij})) \\
&= 1 - \bm{K}^{(0)}_{ij} + 2f(\bm{K}^{(0)}_{ij}) - f(2f(\bm{K}^{(0)}_{ij})) \\
&= 1 - \bm{K}^{(0)}_{ij} + 2\left( h(\bm{C}^{(0)}_{ij}) - h(h(\bm{C}^{(0)}_{ij})) \right).
\end{aligned}
$$
As $\bm{K}^{(0)}_{ij} \rightarrow 1-$, we have $\bm{C}^{(0)}_{ij} \rightarrow 1-$ and $h(\bm{C}_{ij}^{(l-1)}) \rightarrow 1-$. Thus $\lim_{\bm{K}^{(0)}_{ij} \rightarrow 1-} \lambda_{\text{dag3}} - \lambda_{\text{dag2}} = 0$.
Moreover, take the derivative of $\lambda_{\text{dag3}} - \lambda_{\text{dag2}}$ w.r.t. $\bm{K}^{(0)}_{ij}$, we have:
$$
\begin{aligned}
\frac{\partial (\lambda_{\text{dag3}} - \lambda_{\text{dag2}})}{\partial \bm{K}^{(0)}_{ij}} &= -1 + 2 h'(\bm{K}^{(0)}_{ij}) (1 - h'(h(\bm{K}^{(0)}_{ij}))) \\
&< -1 + 2 \cdot 1 \cdot (1 - \frac{1}{2}) = 0.
\end{aligned}
$$
The inequality considers the fact that $h' \in [\frac{1}{2}, 1)$. Therefore, $\lambda_{\text{dag3}} - \lambda_{\text{dag2}}$ is a monotonically decreasing function w.r.t. $\bm{K}^{(0)}_{ij}$ on $[0, 1)$. Thus, we have $\lambda_{\text{dag3}} - \lambda_{\text{dag2}} > 0$ on $[0, 1)$.

In conclusion, we have $\lambda_{\text{dag3}} > \lambda_{\text{dag2}} > \lambda_{\text{dag1}}$. Therefore, the convergence rates are ranked as: DAG\#3 $>$ DAG\#2 $>$ DAG\#1.

\section{Coefficient of Multiple Correlation $R$} \label{appendix:corr_multi}

Our correlation coefficient $R$ between the accuracy and the joint of both $\bar{d}$ and $\bar{m}$ is calculated as blow:

\begin{equation}
\begin{aligned}
    R^{2} &= \bm{c}^{\top} \bm{R}_{\bar{d}, \bar{m}}^{-1} \bm{c} \\
    \bm{c} &= [r_{\bar{d}, y}, r_{\bar{m}, y}]^T \\
    \bm{R}_{\bar{d}, \bar{m}} &= \begin{bmatrix}
r_{x_{\bar{d}} x_{\bar{d}}} & r_{x_{\bar{d}} x_{\bar{m}}} \\
r_{x_{\bar{m}} x_{\bar{d}}} & r_{x_{\bar{m}} x_{\bar{m}}}
\end{bmatrix}
\end{aligned}
\end{equation}

\section{Proof of Theorem~\ref{thm:linear_convergence}}

\begin{customthm}{\ref{thm:linear_convergence}}[Linear Convergence of DAG]
    At $k$-th iteration of gradient descent on $N$ training data samples, we have a network parameterized by $\bm{\theta}(k)$ with its output as $\bm{u}(k)$. The network's connectivity pattern can be formulated as a DAG of $H$ nodes and $P$ end-to-end paths.
    With MSE loss $\mathcal{L}(\bm{\theta}(k))=\frac{1}{2}\|\bm{y}-\bm{u}(k)\|_{2}^{2}$, suppose the learning rate {$\eta=O\left(\frac{\lambda_{\min}\left(\bm{K}^{(H)}\right)}{(NP)^{2}} 2^{O(H)}\right)$ and the number of neurons (width) per layer $m=\Omega\left(\max \left\{\frac{(NP_H)^{4}}{\lambda_{\min}^{4}\left(\bm{K}^{(H)}\right)},\frac{NP_HH}{\delta}, \frac{(NP_H)^{2} \log \left(\frac{HN}{\delta}\right)2^{O(H)}}{\lambda_{\min}^{2}\left(\bm{K}^{(H)}\right)}\right\}\right)$}, we have
    \begin{equation}
        \|\mathbf{y}-\mathbf{u}(k)\|_{2}^{2} \leq\left(1-\frac{\eta \lambda_{\min}(\bm{K}^{(H)})}{2}\right)^{k}\|\mathbf{y}-\mathbf{u}(0)\|_{2}^{2}
    \end{equation}
\end{customthm}

\begin{proof}[Proof Sketch of Theorem~\ref{thm:linear_convergence}]
\label{appendix:proof_sketch_thm1}
The main purpose of this theorem is to extend the convergence of MLP and ResNet studied in~\cite{du2019gradient} to general DAG networks.
One can first decompose the induction hypothesis:
\begin{condition}\label{cond:linear_converge}
        At the $k$-th iteration, we have \begin{align*}
        \norm{\vect{y}-\vect{u}(k)}_2^2 \le (1-\frac{\eta \lambda_0}{2})^{k} \norm{\vect{y}-\vect{u}(0)}_2^2.
        \end{align*}
\end{condition}
into:
\begin{equation}
\begin{aligned}
	&\norm{\vect{y}-\vect{u}(k+1)}_2^2 \\
	= &\norm{\vect{y}-\vect{u}(k) - (\vect{u}(k+1)-\vect{u}(k))}_2^2 \\
	= & \norm{\vect{y}-\vect{u}(k)}_2^2 - 2 \left(\vect{y}-\vect{u}(k)\right)^\top \left(\vect{u}(k+1)-\vect{u}(k)\right) + \norm{\vect{u}(k+1)-\vect{u}(k)}_2^2 \\
	= &\norm{\vect{y}-\vect{u}(k)}_2^2 - 2 \left(\vect{y}-\vect{u}(k)\right)^\top\vect{I}_1(k) -2\left(\vect{y}-\vect{u}(k)\right)^\top {\vect{I}_2(k)} +  \norm{\vect{u}(k+1)-\vect{u}(k)}_2^2 \\
    \le & \left(1-\eta \lambda_{\min}\left(\mat{G}^{(H)}(k)\right)\right)\norm{\vect{y}-\vect{u}(k)}_2^2 -2\left(\vect{y}-\vect{u}(k)\right)^\top\vect{I}_2(k)+  \norm{\vect{u}(k+1)-\vect{u}(k)}_2^2,
	\label{eqn:loss_expansion}
\end{aligned}
\end{equation}
where $\mat{G}^{(H)}(k) = \sum_{s=0}^{H-1} \mat{G}^{(s,H)}(k)$ as the NTK of incoming edges into the last node, composed by the NTK of each edge.

This induction implies that the convergence rate of gradient descent dynamics of DAG networks with an MSE loss is linear.
To prove the induction in Theorem~\ref{thm:linear_convergence}, three core steps are:
\begin{itemize}[align=left,leftmargin=*]
    \item[Step 1.] Show at initialization, $\lambda_{\min}\left(\mat{G}^{(H)}(0)\right) \ge \frac{3\lambda_{\min}(\bm{K}^{(H)})}{4}$ when $m$ is sufficiently large.
    \item[Step 2.] Show that during gradient descent, $\norm{\mat{G}^{(H)}(k)-\mat{G}^{(H)}(0)}_2 \le \frac{1}{4}\lambda_{\min}(\bm{K}^{(H)})$ when $m$ is sufficiently large.
    \item[Step 3.] Show the last two terms, both $-2\left(\vect{y}-\vect{u}(k)\right)^\top\vect{I}_2(k)$ and $\norm{\vect{u}(k+1)-\vect{u}(k)}^2_2$ are proportional to $\eta^2 \norm{\vect{y}-\vect{u}(k)}_2^2$, and if we set $\eta$ sufficiently small, they are smaller than $\eta \lambda_{\min}\left(\mat{G}^{(H)}(k)\right)\norm{\vect{y}-\vect{u}(k)}_2^2$, leading the MSE loss linearly decrease.
\end{itemize}

In our proof, we mainly focus on deriving the condition on $m$ (the required number neurons per layer) and $\eta$ (learning rate) in our DAG setting, by analyzing $\lambda_{\min}\left(\mat{G}^{(H)}(0)\right)$.
Our proof relies on lemmas below, which are based on first analyzing the gradient on each path and then considering the DAG as a union bound over all $P_H$ number of end-to-end paths.

\end{proof}

We want to point out another benefit of bounding both $\lambda_{\min}\left(\mat{G}^{(H)}(k)\right)$ and the linear convergence rate via $\lambda_{\min}(\bm{K}^{(H)})$. Since we try to find the DAG-specific propagation from $\lambda_{\min}(\bm{K}^{(0)})$ to $\lambda_{\min}(\bm{K}^{(H)})$, and all DAGs share the same $\lambda_{\min}(\bm{K}^{(0)})$, we can more fairly compare the bound of convergence rate of different DAGs.

\setcounter{lemma}{0}

\subsection{Lemmas}

\begin{definition}
The geometric series function:
$$
g_{\alpha}(n)=\sum_{i=0}^{n-1}\alpha^i
$$
\end{definition}

\begin{lemma}[Lemma on Initialization Norms]
	\label{lem:init_norm}
	If $\sigma(\cdot)$ is a $L-$Lipschitz activation and $m = \Omega\left(\frac{N P_H^2 Hg_C(H)^2}{\delta}\right)$, where $C\triangleq c_{\sigma}L\left(2\abs{\sigma(0)}\sqrt{\frac{2}{\pi}}+2L\right)$ and $P_H$ is the total number of end-to-end paths,
	then with probability at least $1-\delta$ over random initialization, for every $h\in[H]$ and $i \in [N]$, we have 
	$0 \le \norm{\bm{X}_i^{(h)}(0)}_2 \le c_{x,0}$, where $c_{x,0} = 2 P_H \max(1, c_{w,0}L) $ and $c_{w,0}$ is a universal constant.
\end{lemma}

We follow the proof sketch described in Section~\ref{appendix:proof_sketch_thm1}.
We first analyze the spectral property of $\mat{G}^{(H)}(0)$ at the initialization phase.
The following lemma lower bounds its least eigenvalue.

\begin{lemma}[Least Eigenvalue at the Initialization]\label{lem:mlp_least_eigen}
If $m = \Omega\left(\frac{(NP_H)^2\log(HN/\delta)2^{O(H)}}{\lambda_{\min}(\bm{K}^{(0)})^2}\right)$, we have \begin{align*}
	\lambda_{\min}(\mat{G}^{(H)}(0)) \ge \frac{3}{4}\lambda_{\min}(\bm{K}^{(0)}).
\end{align*}
\end{lemma}

Next, we characterize how the perturbation on the weight matrices affects the input of each layer.
\begin{lemma}\label{lem:pertubation_of_neuron_res}
    Suppose $\sigma(\cdot)$ is $L$-Lipschitz and for $h\in[H]$, $\norm{\mat{W}^{(s,t)}(0)}_2 \le c_{w,0}\sqrt{m}$, $\norm{\vect{X}^{(t)}(0)}_2 \le c_{x,0}$ and $\norm{\mat{W}^{(s,t)}(k)-\mat{W}^{(s,t)}(0)}_F \le \sqrt{m} R$ for some constant $c_{w,0},c_{x,0} > 0$ and $R\le c_{w,0}$ .
    Then we have
    \begin{align*}
        \norm{\vect{X}^{(t)}(k)-\vect{X}^{(t)}(0)}_2 \le
        2 \prod_{s=1}^t P_{s} \sqrt{c_\sigma} c_{w,0}L + \sum_{s=1}^t P_{s} \sqrt{c_\sigma}Lc_{x,0}R.
    \end{align*}
\end{lemma}

Next, we characterize how the perturbation on the weight matrices affect $\mat{G}^{(H)}$.
\begin{lemma}\label{lem:close_to_init_small_perturbation_res_smooth}
        Suppose $\sigma(\cdot)$ is differentiable, $L-$Lipschitz and $\beta-$smooth. Suppose for $s,t\in[H]$, $\norm{\mat{W}^{(s,t)}(0)}_2\le c_{w,0}\sqrt{m}$, $\norm{\vect{a}(0)}_2\le a_{2,0}\sqrt{m}$, $\norm{\vect{a}(0)}_4\le a_{4,0}m^{1/4}$ , $0\le\norm{\vect{x}^{(t)}(0)}_2 \le c_{x,0}$, if $\norm{\mat{W}^{(s,t)}(k)-\mat{W}^{(s,t)}(0)}_F, \norm{\vect{a}(k)-\vect{a}(0)}_2\le \sqrt{m}R$ where $R \le c \lambda_0(P_HN)^{-1}$ and $R\le c$ for some small constant $c$, we have  \begin{align*}
        \norm{\mat{G}^{(H)}(k) - \mat{G}^{(H)}(0)}_2 \le \frac{\lambda_0}{2}.
        \end{align*}
\end{lemma}

\begin{lemma}\label{lem:dist_from_init_resnet}
        If Condition~\ref{cond:linear_converge} holds for $1,\ldots,k$, we have for any $k' \in [k+1]$ and any $s,t \in [H]$
        \begin{align*}
        &\norm{\mat{W}^{(s,t)}(k')-\mat{W}^{(s,t)}(0)}_F, \norm{\vect{a}(k')-\vect{a}(0)}_2\le  R\sqrt{m},\\
        &\norm{\mat{W}^{(s,t)}(k')-\mat{W}^{(s,t)}(k'-1)}_F,\norm{\vect{a}(k')-\vect{a}(k'-1)}_2\le \eta Q(k'-1),
        \end{align*}where $R= P_H \frac{16 \sqrt{c_\sigma} c_{x,0}a_{2,0}Le^{2\sqrt{c_\sigma}c_{w,0}L} \sqrt{N} \norm{\vect{y}-\vect{u}(0)}_2}{\lambda_0\sqrt{m}} <c$ for some small constant $c$ and \\$ Q(k')=  4 P_H \sqrt{c_\sigma} c_{x,0}a_{2,0}Le^{2 \sqrt{c_\sigma} c_{w,0}L}\sqrt{N} \norm{\vect{y}-\vect{u}(k')}_2$.
\end{lemma}

The next lemma bounds the $\vect{I}_2$ term.
\begin{lemma}\label{lem:resnet_I2}
If Condition~\ref{cond:linear_converge} holds for $1,\ldots,k$ and $\eta\le c\lambda_0(NP_H)^{-2}$ for some small constant $c$, we have \begin{align*}
\norm{\vect{I}_2(k)}_2 \le \frac{1}{8}\eta \lambda_0 \norm{\vect{y}-\vect{u}(k)}_2.
\end{align*}
\end{lemma}

Next we bound the quadratic term.
\begin{lemma}\label{lem:quadratic_resnet}
If Condition~\ref{cond:linear_converge} holds for $1,\ldots,k$ and $\eta\le c\lambda_0(NP_H)^{-2}$ for some small constant $c$, we have
$\norm{\vect{u}(k+1)-\vect{u}(k)}_2^2\le \frac{1}{8}\eta \lambda_0 \norm{\vect{y}-\vect{u}(k)}_2^2$.
\end{lemma}

\subsection{Proofs of Lemmas}

\begin{proof}[Proof of Lemma~\ref{lem:init_norm}]
	We will bound $\norm{\bm{X}_i^{(t)}(0)}_2$ by induction on layers. 
	For the first layer, when $\vect{W}^{(0,1)}$ is a parameterized linear operator, we can calculate:
\begin{equation*}
\begin{aligned}
\expect\left[\norm{\vect{X}_i^{(1)}(0)}_2^2\right] = &c_{\sigma} \expect\left[
\relu{\vect{W}^{(0,1)}(0)^\top \vect{X}_i}^2
\right] \\
= & c_{\sigma}\expect_{X\sim N(0,1)} \sigma(X)^2\\
=& 1.\\
\variance\left[\norm{\vect{X}_i^{(1)}(0)}_2^2\right] = &\frac{c_{\sigma}^2}{m} \variance\left[\relu{\vect{W}^{(0,1)}(0)^\top \vect{X}_i(0)}^2\right] \\
\le & \frac{c_{\sigma}^2}{m} \expect_{X\sim N(0,1)} \sigma(X)^4 \\
\le & \frac{c_{\sigma}^2}{m} \expect\left[\left(\abs{\sigma(0)}+
L\abs{\vect{W}^{(0,1)}(0)^\top \vect{X}_i}\right)^4\right] \\
\le &\frac{C_2}{m},
\end{aligned}
\end{equation*}
where $C_2\triangleq \sigma(0)^4+4\abs{\sigma(0)}^3L\sqrt{2/\pi}+6\sigma(0)^2L^2+8\abs{\sigma(0)}L^3\sqrt{2/\pi}+32L^4$.
Considering $\vect{W}^{(0,1)}$ could be a zero or an identity mapping, We have with probability at least $1-\frac{\delta}{n}$,
\begin{align*}
0 \le \norm{\vect{X}_i^{(1)}(0)}_2 \le 2.
\end{align*}

By definition we have for $2\le t\le H$,
\begin{align*}
0 \le \norm{\vect{X}_i^{(t)}(0)}_2 \le \sum_{s=0}^{P_t} \norm{\sqrt{\frac{c_\sigma}{m}}\relu{\mat{W}^{(s,t)}(0)\vect{X}_i^{(s)}(0) }}_2.
\end{align*}

Thus we have:
\begin{equation*}
\begin{aligned}
\variance\left[\norm{\vect{X}_i^{(t)}(0)}_2^2\right]
&= \variance\left[\norm{\sum_{s=0}^{P_t} \sqrt{\frac{c_\sigma}{m}}\relu{\mat{W}^{(s,t)}(0)\vect{X}_i^{(s)}(0) }}_2^2\right] \\
&\leq \sum_{s_1=0}^{P_t} \sum_{s_2=0}^{P_t} \frac{C_2}{m} = P_t^2 \frac{C_2}{m}.
\end{aligned}
\end{equation*}

Applying Chebyshev's inequality and plugging in our assumption on $m$, we have with probability $1-\frac{\delta}{NH}$ over $\mat{W}^{(s, t)}$,
\[ \abs{ \norm{\vect{X}_i^{(t)}(0)}_2^2-\expect \norm{\vect{X}_i^{(t)}(0)}_2^2 } \le \frac{1}{2g_C(H)}. \]

Since $\mat{W}^{(s,t)}$ could be linear, identity, or zero mapping, we have:
\[\norm{\sqrt{\frac{c_\sigma}{m}}\relu{\mat{W}^{(s,t)}(0)\vect{X}_i^{(s)}(0) }}_2 \le \norm{\vect{X}_i^{(s)}(0)}_2 \max(1, c_{w,0}L). \]
Thus
\[ 0 \le \norm{\vect{X}_i^{(t)}(0)}_2 \le \max(1, c_{w,0}L) \sum_{s=0}^{P_t} \norm{\vect{X}_i^{(s)}(0)}_2,\]
which implies
\[ 0 \le \norm{\vect{X}^{(t)}(0)}_2 \le 2 P_H \max(1, c_{w,0}L).\]
Choosing $c_{x,0}=P_H c_{w,0}L$ and using union bounds over $[n]$, we prove the lemma.

\end{proof}

\begin{proof}[Proof of Lemma~\ref{lem:mlp_least_eigen}]
The purpose of this proof is to study how large $m$ is needed to ensure the randomly generated Gram matrices ($\bm{X}^\top \bm{X}$) is close to the population Gram matrices ($\bm{K}$) in networks of general DAG connectivity patterns, via extending the general framework in Section E in~\cite{du2019gradient}.
For complete definitions, we refer readers to Section E.1 in~\cite{du2019gradient}.

\begin{theorem}\label{thm:main_general_framework}
With probability $1- \delta$ over the $\left\{\mat{W}^{(s,t)}\right\}_{s,t}$, suppose $m=\Omega\left(\frac{(NP_H)^{2} \log \left(\frac{HN}{\delta}\right)2^{O(H)}}{\lambda_{\min}^{2}\left(\bm{K}^{(H)}\right)}\right)$,
for any $ 1\le t \le H-1, 1\le i,j\le N,$
\begin{equation}
    \label{eqn:K_bound}
    \norm{\frac{1}{m}\sum\limits_{\alpha=1}^m(\vect{X}^{(t),(\alpha)}_i)^\top \vect{X}^{(t),(\alpha)}_j-\mat{K}^{(t)}_{ij}}_\infty\le \mathcal{E}\sqrt{\frac{\log (HN/\delta)}{m}}
\end{equation}
and any $t \in [H-1], \forall 1\le i\le N,$
\begin{equation}
    \label{eqn:b_bound}
    \norm{\frac{1}{m}\sum_{\alpha=1}^m \vect{X}^{(t),(\alpha)}_i-\vect{b}^{(t)}_i}_\infty\le \mathcal{E} \sqrt{\frac{\log (HN/\delta)}{m}}
\end{equation}
The error constant $\error$ satisfies there exists an absolute constant $C>0$ such that

\resizebox{1.\textwidth}{!}{
\begin{minipage}{\linewidth}
\begin{align*}
    \error\le C \left(\prod_{t=1}^{H-1} \left(
            |P_0^{(t)}| + \sum_{h' \in P_1^{(t)}} \left(\rhobound\wbound_{(t)}+2CB\wbound_{(t)}+2C \sqrt{\wbound_{(t)}M}\right)
        \right) \right)\times
    \max\{\wbound_{(t)}\sqrt{(1+ C^2)M^2},\sqrt{C^2M}\}
\end{align*}
\end{minipage}
}

where
$P_0^{(t)}$ and $P_1^{(t)}$ are sets of indices of nodes with skip-connection and parameterized layers into $\vect{X}^{(s)}$, respectively.
We have $M = O(1), B = O(1), C = O(1), \Lambda = O(1), \wbound_{(t)} = O(1)$, and they are defined by:
\begin{itemize}[leftmargin=*]
    \item $\linsub^{(s,t)}\left(\mat{K}^{(s)}\right) = \expect_{\mat{W}}\left[ \mat{W}^{(s,t)}\mat{K}^{(s)}(\mat{W}^{(s,t)})^\top \right]$.
    \item $M=1+100\max_{i,j,s,t}|\linsub^{(s,t)}(\mat{K}^{(s)}_{ij})|$,
    \item $B=1+100\max_{i,t}|\vect{b}^{(t)}_{i}|$,
    \item $C=\abs{\activate(0)}+\sup_{x\in \mathbb{R}}\abs{\activate'(x)}$,
    \item $\rhobound$ is a constant that only depends on $\rho$,
    \item $\wbound_{(t)}=1+\max_{s}\|\mathcal{W}^{(s,t)}\|_{L^\infty\rightarrow L^\infty}$.
\end{itemize}
\end{theorem}

\begin{proof}[Proof of Theorem~\ref{thm:main_general_framework}]
The proof is by induction.
Suppose that Equation~\eqref{eqn:K_bound} and~\eqref{eqn:b_bound} hold for $1\le s\le t$ with probability at least $1- \frac{t}{H}\delta$, now we want to show the equations holds for $t+1$ with probability at least $1- \delta/H$ conditioned on previous layers satisfying Equation~\eqref{eqn:K_bound} and~\eqref{eqn:b_bound}.
Let $s=t+1$.
recall
\begin{align*}
\vect{X}^{(s),(\alpha)}_i
=
\sum_{h' \in P_0^{(s)}} \vect{X}^{(h')}
+
\sum_{h' \in P_1^{(s)}} \activate^{(h')}\left(\frac{\sum_\beta \mat{W}^{(h',h),(\alpha)}_{(\beta)}\vect{X}^{(h'),(\beta)}_i}{\sqrt{m}}\right),
\end{align*}
Now we get the following formula for the expectation:

\resizebox{0.95\textwidth}{!}{
\begin{minipage}{\linewidth}
\begin{align*}
\expect^{(s)}[\hat{\mat{K}}^{(s)}_{ij}]=& \sum_{h' \in P_0^{(s)}} \hat{\mat{K}}^{(h')}_{ij} +\sum_{h' \in P_1^{(s)}} \expect_{(\vect{U},\vect{V})}\left(\rho(\vect{U})^\top (\hat{\vect{b}}^{(h')}_j)+(\hat{\vect{b}}^{(h')}_i)^\top \rho(\vect{V})+\rho(\vect{U})^\top\rho(\vect{V}))\right)\\
\expect^{(s)}\hat{\vect{b}}_i^{(s)} = & \sum_{h' \in P_0^{(s)}} \hat{\vect{b}}_i^{(h')} + \sum_{h' \in P_1^{(s)}} \expect_{\vect{U}}\rho(\vect{U})
\end{align*}
\end{minipage}
}

with
\begin{align*}
    (\vect{U},\vect{V})\sim N\left(\vect{0},\left(\begin{array}{ccc}
    \linsub^{(h',s)}(\hat{\mat{K}}_{ii}^{(h')}) & \linsub^{(h',s)}( \hat{\mat{K}}_{ij}^{(h')})\\
    \linsub^{(h',s)}( \hat{\mat{K}}_{ji}^{(h')}) & \linsub^{(h',s)}( \hat{\mat{K}}_{jj}^{(h')})\\
    \end{array}\right)\right)
\end{align*}
To bound the differences that determine how the error propagates through layers:
\begin{align*}
\max_{ij}\norm{\expect^{(s)}\hat{\mat{K}}_{ij}^{(s)}- \mat{K}_{ij}^{(s)}}_{\infty} \text{ and }
\max_{i}\norm{\expect^{(s)}\hat{\vect{b}}^{(s)}_{ij}- \vect{b}^{(s)}_{ij}}_{\infty},
\end{align*}
we analyze the error:
\begin{align*}
&\norm{\expect^{(s)}\hat{\mat K}_{ij}^{(s)}- \mat K_{ij}^{(s)}}_{\infty}\\
\le & \sum_{h' \in P_0^{(s)}} \norm{\hat{\mat K}^{(h')}_{ij}-\mat{K}^{(h')}_{ij}}_{\infty}\\
& + \sum_{h' \in P_1^{(s)}} \Large( \norm{\expect_{(\vect{U},\vect{V})\sim \hat{\mat A}}\rho(\vect{U})^\top (\hat{\vect{b}}^{(h')}_j)-\expect_{(\vect{U},\vect{V})\sim \hat{\mat A}}\rho(\vect{U})^\top (\vect{b}^{(h')}_j)}_{\infty}\\
& + \norm{\expect_{(\vect{U},\vect{V})\sim \hat{\mat A}}\rho(\vect{U})^\top (\vect{b}^{(h')}_j)-\expect_{(\vect{U},\vect{V})\sim \mat{A}}\rho(\vect{U})^\top (\vect{b}^{(h')}_j)}_{\infty}\\
&+\norm{\expect_{(\vect{U},\vect{V})\sim \hat{\mat A}}(\hat{\vect{b}}^{(h')}_i)^\top \rho(\vect{V})-\expect_{(\vect{U},\vect{V})\sim \hat{\mat{A}}}((\vect{b}^{(h')}_i)^\top \rho(\vect{V})}_{\infty}\\
&+\norm{\expect_{(\vect{U},\vect{V})\sim \hat{\mat{A}}}(\vect{b}^{(h')}_i)^\top \rho(\vect{V})-\expect_{(\vect{U},\vect{V})\sim \mat{A}}(\vect{b}^{(h')}_i)^\top \rho(\vect{V})}_{\infty}\\
&+\norm{\expect_{(\vect{U},\vect{V})\sim \hat{\mat{A}}}\rho(\vect{U})^\top\rho(\vect{V}))-\expect_{(\vect{U},\vect{V})\sim \mat{A}}\rho(\vect{U})^\top\rho(\vect{V}))}_{\infty} \Large)
\end{align*}
where we define
\begin{align*}
\hat{\mat A}=\left(\begin{array}{ccc}
\linsub^{(h',s)}(\hat{\mat
K}_{ii}^{(h')}) & \linsub^{(h',s)}( \hat{\mat K}_{ij}^{(h')})\\
\linsub^{(h',s)}( \hat{\mat K}_{ji}^{(h')}) & \linsub^{(h',s)}( \hat{\mat K}_{jj}^{(h')})\\
\end{array}\right)
\text{and }
\mat{A}=\left(\begin{array}{ccc}
\linsub^{(h',s)}(\mat
K_{ii}^{(h')}) & \linsub^{(h',s)}(\mat K_{ij}^{(h')})\\
\linsub^{(h',s)}( \mat K_{ji}^{(h')}) & \linsub^{(h',s)}( \mat K_{jj}^{(h')})\\
\end{array}\right)
\end{align*}

By definition, we have
\begin{align*}
\|\mat A-\hat{\mat A}\|_\infty\le &\wbound_{(s)}\max\limits_{ij}\|\hat{\mat K}^{(h')}_{ij}-\mat K^{(h')}_{ij}\|_{\infty}
\end{align*}
We can also estimate other terms
\begin{align*}
&\norm{\expect_{(\vect{U},\vect{V})\sim \hat{\mat A}}\rho(\vect{U})^\top (\hat{\vect{b}}^{(h')}_j)-\expect_{(\vect{U},\vect{V})\sim \hat{ \mat A}}\rho(\vect{U})^\top (\vect{b}^{(h')}_j)}_{\infty}\\
\le &\norm{\expect_{(\vect{U},\vect{V})\sim \hat{\mat A}}\rho(\vect{U})^\top \left(\hat{\vect{b}}^{(h')}_j-\vect{b}^{(h')}_j\right)}_{\infty}\\
\le &C\sqrt{\wbound_{(s)} \max_{ij}\|\hat{\mat K}_{ij}^{(s)}\|_\infty}\max_{i}\norm{\hat{\vect{b}}^{(h')}_{ij}- \vect{b}^{(h')}_{ij}}_{\infty}\\
\le &C \sqrt{\wbound_{(s)}M}\max_{i}\norm{\hat{\vect{b}}^{(h')}_{ij} - \vect{b}^{(h')}_{ij}}_{\infty},
\end{align*}

\begin{align*}
&\norm{\expect_{(\vect{U},\vect{V})\sim \hat{\mat A}}\rho(\vect{U})^\top (\vect b^{(h')}_j) - \expect_{(\vect U,\vect V)\sim \mat A}\rho(\vect U)^\top (\vect{b}^{(h')}_j)}_{\infty}\\
\le &B C\|\mat A-\hat{\mat A}\|_\infty \\
\le &BC \wbound_{(s)} \max\limits_{ij}\|\hat{\mat K}^{(h')}_{ij}-\mat K^{(h')}_{ij}\|_{\infty},
\end{align*}
and
\begin{align*}
&\norm{\expect_{(\vect{U},\vect{V})\sim \hat{\mat A}}\rho(\vect{U})^\top\rho(\vect{V})-\expect_{(\vect{U},\vect{V})\sim \mat A}\rho(\vect{U})^\top\rho(\vect{V})}_{\infty}\\
\le &\rhobound\|\mat A-\hat{\mat A}\|_{\infty}\\
\le &\rhobound\wbound_{(s)}\max\limits_{ij}\|\hat{\mat K}^{(h')}_{ij}-\mat K^{(h')}_{ij}\|_{\infty}.
\end{align*}

Putting these estimates together, we have

\resizebox{1.\textwidth}{!}{
\begin{minipage}{\linewidth}
\begin{align*}
&\max_{ij}\|\expect^{(s)}\hat{\mat K}_{ij}^{(s)}- \mat K_{ij}^{(s)}\|_\infty\\
& \le \sum_{h' \in P_0^{(s)}} \left(\max\limits_{ij}\|\hat{\mat K}^{(h')}_{ij}-\mat K^{(h')}_{ij}\|_{\infty}\vee\max_{i}\|\hat{\vect b}^{(h')}_{ij}- \vect b^{(h')}_{ij}\|_\infty\right) \\
&+ \sum_{h' \in P_1^{(s)}} \left(\rhobound\wbound_{(s)}+2CB\wbound_{(s)}+2C \sqrt{\wbound_{(s)}M}\right)\left(\max\limits_{ij}\|\hat{\mat K}^{(h')}_{ij}-\mat K^{(h')}_{ij}\|_{\infty}\vee\max_{i}\|\hat{\vect b}^{(h')}_{ij}- \vect b^{(h')}_{ij}\|_\infty\right)
\end{align*}
\end{minipage}
}

and

\resizebox{1.\textwidth}{!}{
\begin{minipage}{\linewidth}
\begin{align*}
\max_{i}\norm{\expect^{(s)}\hat{\vect b}^{(s)}_{ij}- \vect b^{(s)}_{ij}}_{\infty}
& \le \sum_{h' \in P_0^{(s)}} \left(\max_{ij}\|\hat{\mat K}^{(h')}_{ij}-\mat K^{(h')}_{ij}\|_\infty\vee\max_{i}\|\hat{\vect b}^{(h')}_{ij}- \vect{b}^{(h')}_{ij}\|_\infty\right) \\
& + \sum_{h' \in P_1^{(s)}} (\rhobound\wbound_{(s)})\left(\max_{ij}\|\hat{\mat K}^{(h')}_{ij}-\mat K^{(h')}_{ij}\|_\infty\vee\max_{i}\|\hat{\vect b}^{(h')}_{ij}- \vect{b}^{(h')}_{ij}\|_\infty\right).
\end{align*}
\end{minipage}
}

These two bounds imply the theorem.

\end{proof}

\begin{remark}
Combing Theorem~\ref{thm:main_general_framework} and Lemma E.1 in~\cite{du2019gradient} and standard matrix perturbation bound directly have Lemma~\ref{lem:mlp_least_eigen}.
\end{remark}

\end{proof}

\begin{proof}[Proof of Lemma~\ref{lem:pertubation_of_neuron_res}]
	We prove this lemma by induction.
	Our induction hypothesis is
\begin{align*}
	\norm{\vect{X}^{(t)}(k)-\vect{X}^{(t)}(0)}_2 \le g(t)  ,
\end{align*}
where
\begin{align*}
    g(t) =
    \sqrt{c_\sigma} P_t \left( 2 g(t-1) c_{w,0}L + L c_{x,0}R \right).
\end{align*}
For $h=1$, we have
\begin{align*}
	\norm{\vect{X}^{(1)}(k)-\vect{X}^{(1)}(0)}_2&\le \sqrt{\frac{c_{\sigma}}{m}} \norm{\relu{\mat{W}^{(0,1)}(k) \vect{X}} -\relu{\mat{W}^{(0,1)}(0) \vect{X}}}_2\\
&	\le \sqrt{\frac{c_{\sigma}}{m}}\norm{\mat{W}^{(0,1)}(k)-\mat{W}^{(0,1)}(0)}_F \le \sqrt{c_{\sigma}}LR,
\end{align*}
which implies $g(1)=\sqrt{c_{\sigma}}LR$.
For $2\le t\le H$, following the proof of Lemma C.3. in~\cite{du2019gradient}, we have
\begin{align*}
\norm{\vect{X}^{(t)}(k)-\vect{X}^{(t)}(0)}_2 &= \sqrt{\frac{c_\sigma}{m}} \norm{ \sum_{s=0}^{P_t} \relu{\mat{W}^{(s,t)}(k) \vect{X}^{(s)}(k)} -\relu{\mat{W}^{(s,t)}(0)\vect{x}^{(s)}(0) }}_2 \\
&\leq \sqrt{\frac{c_\sigma}{m}} \sum_{s=0}^{P_t} \norm{ \relu{\mat{W}^{(s,t)}(k) \vect{X}^{(s)}(k)} -\relu{\mat{W}^{(s,t)}(0)\vect{x}^{(s)}(0) }}_2 \\
&\leq P_t \left( 2 \sqrt{c_\sigma} c_{w,0}Lg(t-1) + \sqrt{c_\sigma} L c_{x,0}R \right).
\end{align*}
Lastly, simple calculations show $g(t) \leq 2 \prod_{s=1}^t P_{s} \sqrt{c_\sigma}c_{w,0}L + \sum_{s=1}^t P_{s} \sqrt{c_\sigma}Lc_{x,0}R$.
\end{proof}

\begin{proof}[Proof of Lemma~\ref{lem:close_to_init_small_perturbation_res_smooth}]
	Following the proof of Lemma C.4. in~\cite{du2019gradient}, we can obtain
	\begin{align*}
	&\abs{\mat{G}_{i,j}^{(H)}(k)-\mat{G}_{i,j}^{(H)}(0)} \leq  P_H\left(I_1^{i,j} + I_2^{i,j}+I_3^{i,j}\right) .
	\end{align*}
Therefore we can bound the perturbation
\begin{align*}
	&\norm{\mat{G}^{(H)}(k) - \mat{G}^{(H)}(0)}_F = \sqrt{\sum_{(i,j)}^{{N,N}} \abs{\mat{G}_{i,j}^{(H)}(k)-\mat{G}_{i,j}^{(H)}(0)}^2} \\
	&\leq P_H L^2NR \left[3 c_{x,0} c_xa_{2,0}^2+2\beta c_{x,0}^2 \left(2c_xc_{w,0}+c_{x,0}\right)a_{4,0}^2+12c_{x,0}^2 a_{2,0}\right], \\
\end{align*}
where $c_x \triangleq 2 \prod_{s=1}^t P_{s} \sqrt{c_\sigma} c_{w,0} \frac{L}{R} + \sum_{s=1}^t P_{s} \sqrt{c_\sigma} Lc_{x,0}$.
Plugging in the bound on $R$, we have the desired result.
	
\end{proof}

\begin{proof}[Proof of Lemma~\ref{lem:dist_from_init_resnet}]
	We will prove this corollary by induction. The induction hypothesis is
	\begin{align*}
	\norm{\mat{W}^{(s,t)}(k')-\mat{W}^{(s,t)}(0)}_F &\le \sum_{k''=0}^{k'-1} (1-\frac{\eta \lambda_0}{2})^{k''/2}\frac{1}{4}\eta \lambda_0 R\sqrt{m}\le R\sqrt{m}, k'\in [k+1],\\
	\norm{\vect{a}(k')-\vect{a}(0)}_2 &\le \sum_{k''=0}^{k'-1} (1-\frac{\eta \lambda_0}{2})^{k''/2}\frac{1}{4}\eta \lambda_0 R\sqrt{m}\le R\sqrt{m}, k'\in [k+1].
	\end{align*}
	
	First it is easy to see it holds for $k''=0$. Now suppose it holds for $k''=0,\ldots,k'$, we consider $k''=k'+1$. Applying union bound over $P_H$ number of end-to-end paths, we have
	\begin{align*} 
	&\norm{\mat{W}^{(s,t)}(k'+1)-\mat{W}^{(s,t)}(k')}_F\\
	\le &\eta P_t L \sqrt{\frac{c_\sigma}{m}} \norm{\vect{a}}_2 \sum_{i=1}^{N}\abs{y_i-u_i(k')}\norm{\vect{X}^{(s-1)}_i(k')}_2 \prod_{s'=t+1}^H \norm{\mat{I}+\lambda_0^{3/2}\sqrt{\frac{c_\sigma}{m}}\mat{J}_i^{(s',s)}(k')\mat{W}^{(s,s')}(k') }_2\\
	\le&  2\eta P_t \sqrt{c_\sigma} c_{x,0}La_{2,0}e^{2 \sqrt{c_\sigma} c_{w,0}L}\sqrt{N} \norm{\vect{y}-\vect{u}(k')}_2\\
	=& \eta P_t Q(k')\\
	\le& P_H (1-\frac{\eta \lambda_0}{2})^{k'/2}\frac{1}{4}\eta \lambda_0 R\sqrt{m},\\
	\end{align*}
	where
	\begin{equation*}
	    \mat{J}_i^{(s',s)} \triangleq \diag\left(
\rho' \left((\vect{W}_{1}^{(s',s)})^\top \vect{X}_i^{(s')}\right) , \ldots,
\rho' \left((\vect{W}_{m}^{(s',s)})^\top \vect{X}_i^{(s')}\right)
\right) \in \mathbb{R}^{m \times m}
	\end{equation*}
are the derivative matrices induced by the activation function.
Similarly, we have
\begin{align*}
	\norm{\vect{a}(k'+1)-\vect{a}(k')}_2 \le& 2\eta c_{x,0} \sum_{i=1}^{N}\abs{y_i-u(k')}\\
	\le& \eta Q(k')\\
	\le& (1-\frac{\eta \lambda_0}{2})^{k'/2}\frac{1}{4}\eta \lambda_0 R\sqrt{m}.
	\end{align*}
	Thus
	\begin{align*}
	&\norm{\mat{W}^{(s,t)}(k'+1)-\mat{W}^{(s,t)}(0)}_F\\
	\le& \norm{\mat{W}^{(s,t)}(k'+1)-\mat{W}^{(s,t)}(k')}_F+\norm{\mat{W}^{(s,t)}(k')-\mat{W}^{(s,t)}(0)}_F\\
	\le&\sum_{k''=0}^{k'} \eta (1-\frac{\eta \lambda_0}{2})^{k''/2}\frac{1}{4}\eta \lambda_0 R\sqrt{m}.\\
	\end{align*}
		Similarly,
	\begin{align*}
	&\norm{\vect{a}(k'+1)-\vect{a}(0)}_2\\
	\le&\sum_{k''=0}^{k'} \eta (1-\frac{\eta \lambda_0}{2})^{k''/2}\frac{1}{4}\eta \lambda_0 R\sqrt{m}.\\
	\end{align*}
\end{proof}

\begin{proof}[Proof of Lemma~\ref{lem:resnet_I2}]

We apply union bound over $P_H$ number of end-to-end paths.
\begin{align*}
& \norm{\mathcal{L}'^{(s,t)}(k)}_F \triangleq \norm{ \frac{\partial \mathcal{L}}{\partial \vect{W}^{(s,t)}(k)} }_F \\
= & P_t
\norm{\sqrt{\frac{c_\sigma}{m}}
	\sum_{i=1}^{N}(y_i-u_i(k))\vect{X}_i^{(s)}(k)  \cdot
	\left[\vect{a}(k)^\top \prod_{l=t+1}^{H}\left(\mat{I}+\sqrt{\frac{c_\sigma}{m}}\mat{J}_i^{(t,l)}(k)\mat{W}^{(t,l)}(k) \right)\mat{J}_i^{(t,l)}(k) \right]}_F \\
\le & P_t L \sqrt{\frac{c_\sigma}{m}} \norm{\vect{a}(k)}_2 \sum_{i=1}^{N}\abs{y_i-u_i(k)}\norm{\vect{X}^{(s)}(k)}_2 \prod_{l=t+1}^H \norm{\mat{I}+\sqrt{\frac{c_\sigma}{m}}\mat{J}_i^{(t,l)}(k)\mat{W}^{(t,l)}(k) }_2.
\end{align*}
We have bounded the RHS in the proof for Lemma~\ref{lem:dist_from_init_resnet}, thus
\begin{align*}
	\norm{\mathcal{L}'^{(s,t)}(k)}_F \le \lambda_0 Q(k).
\end{align*}
Denote $\params(k,k')=\params(k)-k' \mathcal{L}'(\params(k))$, we have
\begin{align*}
	&\norm{  u'^{(s,t)}_i\left(\params(k)\right) - u'^{(s,t)}_i\left(\params(k,k')\right)}_F=\\
	& P_t \sqrt{\frac{c_\sigma}{m}}\norm{ \vect{X}_i^{(s)}(k)\vect{a}(k)^\top \prod_{l=t+1}^{H}\left(\mat{I}+\sqrt{\frac{c_\sigma}{m}}\mat{J}_i^{(s,l)}(k)\mat{W}^{(s,l)}(k) \right)\mat{J}_i^{(s,t)}(k) \right.\\
	&\left.-\vect{X}_i^{(s)}(k,k')\vect{a}(k,k')^\top	\prod_{l=t+1}^{H}\left(\mat{I}+\sqrt{\frac{c_\sigma}{m}}\mat{J}_i^{(s,l)}(k,k')\mat{W}^{(s,l)}(k,k') \right)\mat{J}_i^{(s,t)}(k,k')}_F.
\end{align*}
Through standard calculations, we have
\begin{align*}
	\norm{	\mat{W}^{(s,l)}(k)-\mat{W}^{(s,l)}(k,s)}_F \le &\eta Q(k),\\
		\norm{	\vect{a}(k)-\vect{a}(k,s)}_F \le &\eta Q(k),\\
	\norm{	\vect{X}_i^{(s)}(k)-\vect{X}_i^{(s)}(k,s)}_F \le &\eta c_x \frac{Q(k)}{\sqrt{m}},\\
	\norm{	\mat{J}^{(s,l)}(k)-\mat{J}^{(s,l)}(k,s)}_F \le &2\left(c_{x,0}+c_{w,0}c_x\right)\eta \beta   Q(k).
\end{align*}
Thus we have

\resizebox{1.\textwidth}{!}{
\begin{minipage}{\linewidth}
\begin{align*}
	&\norm{  u'^{(t)}_i\left(\params(k)\right) - u'^{(t)}_i\left(\params(k,k')\right)}_F \\
	\le & 4 c_{x,0}La_{2,0}e^{2Lc_{w,0}}\eta Q(k) \sqrt{\frac{c_\sigma}{m}}\left(\frac{c_x}{c_{x,0}}+\frac{2}{L}\left(c_{x,0}+c_{w,0}c_x\right)\beta \sqrt{m}+4c_{w,0}\left(c_{x,0}+c_{w,0}c_x\right)\beta+L+1 \right)\\
	\le & 32 c_\sigma c_{x,0}a_{2,0}e^{2Lc_{w,0}} \left(c_{x,0}+c_{w,0}c_x\right) \beta \eta Q(k).
\end{align*}
\end{minipage}
}

Thus:
\begin{align*}
	\abs{I^i_2}\le 32 c_\sigma c_{x,0}a_{2,0}e^{2Lc_{w,0}} \left(c_{x,0}+c_{w,0}c_x\right) \beta \eta^2  Q(k)^2 \le \frac{1}{8}\eta \lambda_0 \norm{\vect{y}-\vect{u}(k)}_2,
\end{align*}
where we used the bound of $\eta$ and that $\norm{\vect{y}-\vect{u}(0)}_2=O(\sqrt{N})$,.
\end{proof}

\begin{proof}[Proof of Lemma~\ref{lem:quadratic_resnet}]
	\begin{align*}
& \norm{\vect{u}(k+1)-\vect{u}(k)}_2^2 = \sum_{i=1}^{N}\left(\vect{a}(k+1)^\top \vect{X}_i^{(H)}(k+1)-\vect{a}(k)^\top \vect{X}_i^{(H)}(k)\right)^2 \\
&= \sum_{i=1}^{N}\left(\left[\vect{a}(k+1)-\vect{a}(k)\right]^\top \vect{X}_i^{(H)}(k+1)+\vect{a}(k)^\top \left[\vect{X}_i^{(H)}(k+1)-\vect{X}_i^{(H)}(k)\right] \right)^2 \\
&\le 2\norm{\vect{a}(k+1)-\vect{a}(k)}_2^2\sum_{i=1}^{N}\norm{\vect{X}_i^{(H)}(k+1)}_2^2+2\norm{\vect{a}(k)}_2^2\sum_{i=1}^{N}\norm{\vect{X}_i^{(H)}(k+1)-\vect{X}_i^{(H)}(k)}_2^2\\
&\le 8N\eta^2c_{x,0}^2Q(k)^2 + 4 N \left(\eta  a_{2,0}c_xQ(k)\right)^2\\
&\le \frac{1}{8}\eta \lambda_0 \norm{\vect{y}-\vect{u}(k)}_2^2.
\end{align*}
\end{proof}

\section{Proof of Lemma~\ref{lem:full_rank}}

\begin{customlemma}{\ref{lem:full_rank}}[Full Rankness of $\mathbf{K}^{(H)}$]
    Suppose $\sigma(\cdot)$ is analytic and not a polynomial function. If no parallel data points, then $\lambda_{\min} \left(\mathbf{K}^{(H)}\right) > 0$.
\end{customlemma}

\begin{proof}[Proof of Lemma~\ref{lem:full_rank}]
First, we know that $\bm{K}^{(0)}$ is positive definite by our assumption on the data point. The next step is to prove propagation of either skip-connection or linear operation will keep the positive definite property. For a linear operation, the NNGP variance matrix can be expressed as,
$$
\bm{K}_{ij} = \mathbb{E}_{\mat{W} \sim \mathcal{N}(0,\bm{I})}[ \sigma(\mat{W}^\top \mat{X}_i) \sigma(\mat{W}^\top \mat{X}_j)] 
$$
Define the feature map induced by NNGP kernel as $\phi_{\mat{X}_i}(\mat{W}) = \sigma(\mat{W}^\top \mat{X}_i)\mat{X}_i $. Then we will show that $\phi_{\mat{X}_1}(\mat{W}), \ldots \phi_{\mat{X}_N}(\mat{W}) $ are linearly independent. Assume that there are $a_i \in \mathbb{R}$ such that,
$$
\sum_{i=1}^N a_i \phi_{\mat{X}_i} = \sum_{i=1}^N a_i \sigma(\mat{W}^\top \mat{X}_i)\mat{X}_i = 0 
$$
Differentiating the above equation $(N-2)$ times with respect to $\mat{W}$, we then have,
$$
\sum_{i=1}^N (a_i \sigma^{(N)}(\mat{W}^\top \mat{X}_i)) \mat{X}_i^{ \otimes (N-1)} = 0
$$
The linearly independents of $\mat{X}_i^{ \otimes (N-1)}$ can be proven by induction. For $N = 2$, $\mat{X}_1 \mat{X}^\top_1  $ and $\mat{X}_2 \mat{X}^\top_2  $ are linearly independent under the non-parallel assumption. Suppose $\{{\rm vec}({\mat{X}^{\otimes (N-1)}_1} ), \ldots, {\rm vec}({\mat{X}^{\otimes (N-1)}_{N-1}}) \} $ are linearly independent. Then we prove the linearly independent for case of ${\mat{X}}_i^{\otimes N}$ by contradiction. Suppose there exists $\alpha_1, \ldots, \alpha_N \in \mathbb{R}$ not identically 0, such that
$$
\sum_{i=1}^N \alpha_i {\rm vec}(\mat{X}^{\otimes N}_i) = 0
$$
which implies for $j = 1,\dots, d$
$$
\sum_{i=1}^N \alpha_i \mat{X}_{ij} {\rm vec}(\mat{X}^{\otimes (N-1)}_i) = 0
$$
Because we have supposed the linearly independence of $\mat{X}^{\otimes (N-1)}_i$, we know there must exist $j \in [d]$ that $\mat{X}_{ij} \neq 0$ for all $i \in [N]$. However, this is contradict with the assumption that $\alpha_1, \ldots, \alpha_N \neq 0$. Thus we can conclude that $\{\mat{X}^{\otimes (N-1)}_i \}_{i=1}^N$ 
are linearly independent. Finally, we have 
$$
\lambda_{\min}(\bm{K}) > 0
$$

Then we consider the NNGP for node $t$: 

$$
 \bm{K}^{(t)} = \sum_{s \in \mathcal{S}} {\bm{ K}}^{(s,t)} + \sum_{s_1, s_2 \in \mathcal{S}}  {\bm{b}}^{(s_1,t)}  {\bm{b}}^{(s_2,t)}
$$
where $\mathcal{S}$ is the set of nodes in the DAG except $t$. When $\mat{W}^{(s,t)}$ is the skip-connection operator, then we have $\bm{K}^{(s,t)} = \bm{K}^{(s)} $. On the other hand, when $\mat{W}^{(s,t)}$ corresponds to linear transformation, we have $\bm{K}_{ij}^{(s,t)} = \mathbb{E}_{\mat{W} \sim \mathcal{N}(0,\bm{I})}[ \sigma(\mat{W}^\top \bm{z}_i) \sigma(\mat{W}^\top \bm{z}_j)] $, with $\bm{z} \in  \bm{Z}$ and $\bm{Z} = \bm{D}^{1/2} \bm{U}^\top$, where $\bm{U}\bm{D}\bm{U}^\top = \bm{K}^{(m)}$. Thus for each $\bm{K}_{ij}^{(s,t)}$ we can apply the result above and conclude each matrix is positive definite. 

In addition, in each matrix ${\bm{b}}^{(s_1,t)}  {\bm{b}}^{(s_2,t)}$, all entries are the same. Because
$$
\lambda_{\min} \begin{bmatrix}
     \bm{K}_{ii}  & \bm{K}_{ij}\\
     \bm{K}_{ji} & \bm{K}_{jj} 
 \end{bmatrix} = \lambda_{\min} \bigg( \begin{bmatrix} 
     \bm{K}_{ii}  & \bm{K}_{ij}\\
     \bm{K}_{ji} & \bm{K}_{jj} 
 \end{bmatrix} + \begin{bmatrix} 
     \bm{b}_{i} (\bm{b}_{i})^\top  & \bm{b}_{i} (\bm{b}_{j})^\top\\
     \bm{b}_{j} (\bm{b}_{i})^\top & \bm{b}_{j} (\bm{b}_{j})^\top 
 \end{bmatrix} \bigg)
$$
Thus we know that $ \bm{K}^{(t)}$ is positive definite.

Repeat the argument above, we can conclude that $\bm{K}^{(H)}$ is full rankness.

\end{proof}

\section{Convergence vs. DAG Topology}

In Figure~\ref{fig:dag_convergence}, we sample 408 networks from our DAG space ($H = 3$).
We empirically measure how many epochs each network requires to reach 80\% training accuracy as its convergence rate, and plot these epochs with networks' topologies.
We observe that wide and shallow networks show faster convergence.
This observation is aligned with Eq.~\ref{eq:rule}, which states that networks with larger widths ($P_H$, the number of end-to-end paths) and smaller depths ($d_p$, the number of ``Linear + ReLU'' layers on each path) will have large $\lambda_0$.
However, as convergence only partially explains the final test accuracy (with missing factors like generalization and model complexity), Figure~\ref{fig:dag_convergence} does not conflict with Figure~\ref{fig:dag_performance}.

\begin{figure*}[h!]
\includegraphics[scale=0.4]{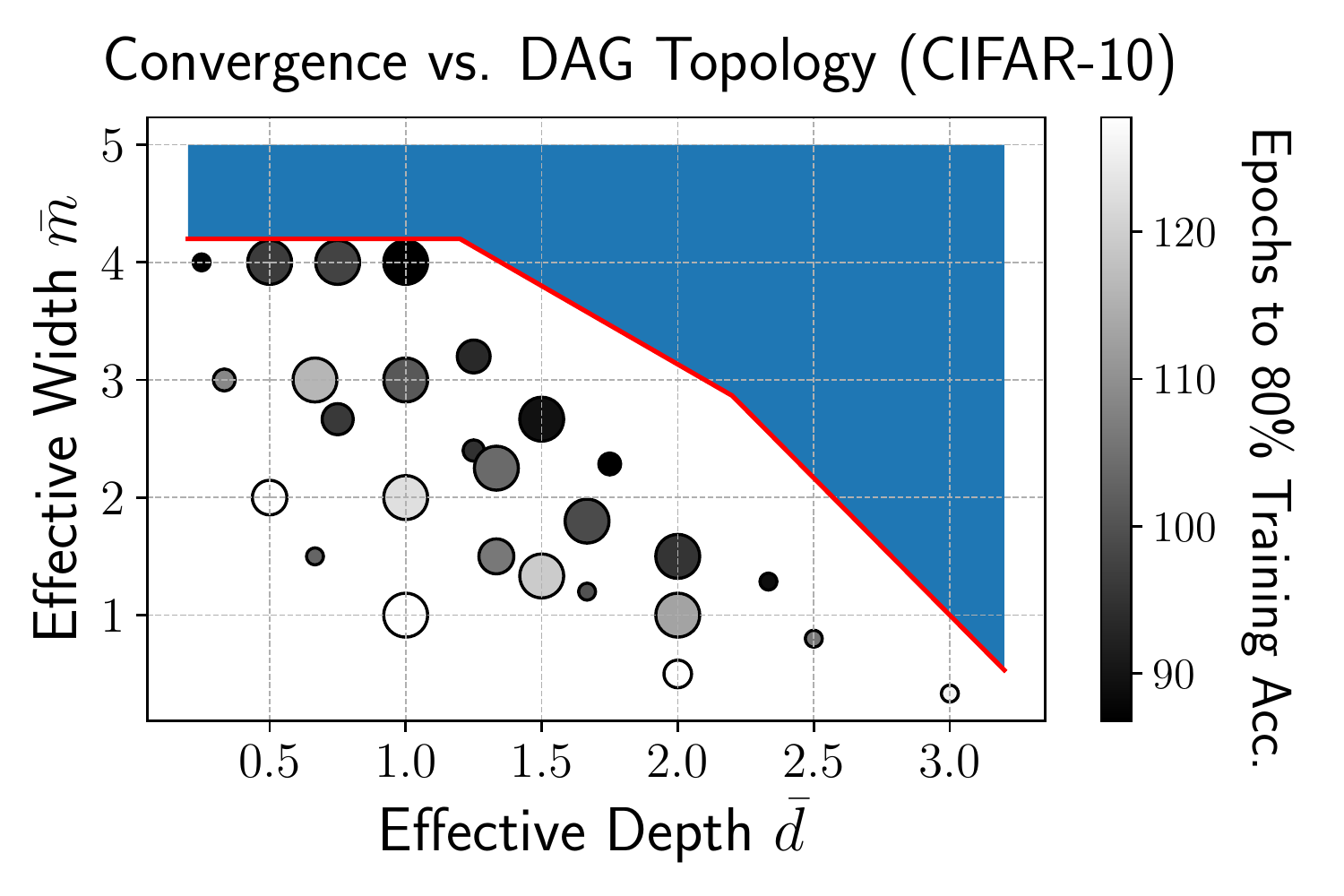}
\centering
\captionsetup{font=small}
\caption{In a complete architecture space, being extremely large or small on either the effective depth ($\bar{d}$) or the effective width ($\bar{m}$) leads a connectivity pattern to have slow convergence (white circles). Each dot represents a subset of architectures of the same $\bar{d}$ and $\bar{m}$. Blue areas indicate invalid DAG regions: a DAG cannot be both deep and wide at the same time.}
\label{fig:dag_convergence}
\end{figure*}

\end{document}